\documentclass[oneside,english]{amsart}
\usepackage[T1]{fontenc}
\usepackage{textcomp}
\usepackage[latin9]{inputenc}
\usepackage{babel}
\usepackage{cprotect}
\usepackage{float}
\usepackage{amstext}
\usepackage{amsthm}
\usepackage{amssymb}
\usepackage{graphicx}
\usepackage{geometry}
\usepackage{wasysym}
\usepackage[]
 {hyperref}

\makeatletter

\newcommand{\lyxmathsym}[1]{\ifmmode\begingroup\def\b@ld{bold}
  \text{\ifx\math@version\b@ld\bfseries\fi#1}\endgroup\else#1\fi}

\floatstyle{ruled}
\newfloat{algorithm}{tbp}{loa}
\providecommand{\algorithmname}{Algorithm}
\floatname{algorithm}{\protect\algorithmname}

\numberwithin{equation}{section}
\numberwithin{figure}{section}
\theoremstyle{plain}
\newtheorem{thm}{\protect\theoremname}
\theoremstyle{definition}
\newtheorem{example}[thm]{\protect\examplename}
\theoremstyle{remark}
\newtheorem{rem}[thm]{\protect\remarkname}
\theoremstyle{definition}
\newtheorem{problem}[thm]{\protect\problemname}
\theoremstyle{definition}
\newtheorem{defn}[thm]{\protect\definitionname}
\theoremstyle{plain}
\newtheorem{cor}[thm]{\protect\corollaryname}
\theoremstyle{plain}
\newtheorem{assumption}[thm]{\protect\assumptionname}
\theoremstyle{remark}
\newtheorem{claim}[thm]{\protect\claimname}

\usepackage[foot]{amsaddr}

\makeatother

\usepackage{listings}
\providecommand{\assumptionname}{Assumption}
\providecommand{\claimname}{Claim}
\providecommand{\corollaryname}{Corollary}
\providecommand{\definitionname}{Definition}
\providecommand{\examplename}{Example}
\providecommand{\problemname}{Problem}
\providecommand{\remarkname}{Remark}
\providecommand{\theoremname}{Theorem}

\begin{document}
\title{Cross-Entropy Games for Language Models: \\
From Implicit Knowledge to General Capability Measures}
\author{Clément Hongler$^{*}$}
\email{clement.hongler@gmail.com}
\address{$^{*}$EPFL}
\author{Andrew Emil}
\email{andrewcemil@gmail.com}
\begin{abstract}
Large Language Models (LLMs) define probability measures on text.
By considering the \emph{implicit knowledge} question of what it means
for an LLM to know such a measure and what it entails algorithmically,
we are naturally led to formulate a series of tasks that go beyond
generative sampling, involving forms of summarization, counterfactual
thinking, anomaly detection, originality search, reverse prompting,
debating, creative solving, etc. These tasks can be formulated as
games based on LLM measures, which we call \emph{Cross-Entropy (Xent)
Games}. 

Xent Games can be single-player or multi-player. They involve cross-entropy
scores and cross-entropy constraints, and can be expressed as simple
computational graphs and programs. We show the Xent Game space is
large enough to contain a wealth of interesting examples, while being
constructible from basic game-theoretic consistency axioms. 

We then discuss how the Xent Game space can be used to measure the
abilities of LLMs. This leads to the construction of \emph{Xent Game
measures}: finite families of Xent Games that can be used as capability
benchmarks, built from a given scope, by extracting a covering measure.
To address the unbounded scope problem associated with the challenge
of measuring general abilities, we propose to explore the space of
Xent Games in a coherent fashion, using ideas inspired by evolutionary
dynamics. 
\end{abstract}

\maketitle
\tableofcontents{}

\section{LLM Measures: Explicit and Implicit Capabilities}

In their so-called pre-trained forms, Large Language Models are trained
to predict the next token in a text, given the previous ones: for
a string of tokens $x_{1},\ldots,x_{n}$ in a vocabulary $\mathcal{V}$,
a model $\mathcal{M}$ computes probabilities $\mathbb{P}_{\mathcal{M}}\left\{ X_{k}=x|x_{<k}\right\} $
for all $k\leq n$ and for all $x\in\mathcal{V}$ (as learned from
a large corpus of text). 

A model $\mathcal{M}$ thus defines a probabilistic model on strings
of tokens $x=\left(x_{1},\ldots,x_{n}\right)$ by 
\[
\mathbb{P}_{\mathcal{M}}\left\{ x\right\} =\prod_{k=1}^{n}\mathbb{P}_{\mathcal{M}}\left\{ X_{k}=x_{k}|x_{<k}\right\} =\exp\left(-\mathcal{S}_{M}\left(x\right)\right)
\]
where the `action' $\mathcal{S}_{\mathcal{M}}$ is defined from the
cross-entropy loss $\ell_{\mathcal{M}}\left(x_{k};x_{<k}\right)=-\log\mathbb{P}\left\{ X_{k}=x_{k}|x_{<k}\right\} $
by 
\[
\mathcal{S}_{\mathcal{M}}\left(x\right)=\sum_{k=1}^{n}\ell_{\mathcal{M}}\left(x_{k};x_{<k}\right).
\]
 being the cross-entropy loss incurred by the model at the $k$-th
prediction. 

Various downstream applications (like agents or chatbots) are then
built from pre-trained models by fine-tuning. Recent years have delivered
such spectacular results that it is now common practice to probe the
`knowledge' of LLMs by assessing their abilities to answer questions
when using them as chatbots or agents; this is in particular how most
of the LLM benchmarking is now performed. 

We call the knowledge LLMs display via direct question-answering the
\emph{explicit knowledge} of the LLM (see Section \ref{subsec:explicit-capabilities}),
in contrast to the \emph{implicit knowledge} (see Section \ref{subsec:implicit-capabilities}),
which consists of all the `relevant' information contained in $\mathcal{S}_{\mathcal{M}}$
(for a suitable definition of `relevant', which we propose). One of
the greatest strengths of LLMs is their ability to deal with \emph{questions}
about themselves in the very space where they operate (i.e. text).
This naturally leads to the question of reflexivity\emph{:} how explicit
can the implicit knowledge be made? Typically, an LLM $\mathcal{M}$
does not have any explicit knowledge of its own measure $\mathcal{S}_{\mathcal{M}}$
and the associated implicit knowledge, or of any LLM measure (see
Section \ref{subsec:implicit-knowledge-vs-explicit-knowledge} below).
We argue that the implicit knowledge of current LLMs extends well
beyond their explicit knowledge (see Section \ref{subsec:explicit-capabilities}),
and that this idea goes a long way, via the introduction of games
(see Section \ref{subsec:implicit-knowledge-and-games}), towards
new applications of LLMs (see Section \ref{sec:xent-game-examples}),
the construction of capability measures for them (see Section \ref{sec:xent-game-measures}),
and via evolution methods, towards general capability measures (see
Section \ref{sec:evolution-in-game-space}).

\subsection{\label{subsec:explicit-capabilities}Explicit Capabilities}

An LLM used as a chatbot answers prompts: given a prompt as an initial
string, it produces as an answer a random sample of a completion,
conditioned on the initial part being the prompt. Since chatbots are
typically fine-tuned to follow instructions, their outputs will typically
be close enough in format to an answer (or to an attempt at answering)
that we can judge the correctness of that answer. Informally speaking,
we call \emph{explicit knowledge }or\emph{ explicit capabilities}
of the LLM the set of questions that the LLM can correctly answer
by this sampling with a reasonably high probability -- we expect
that if we at least perform a large number of samples of answers,
the majority will be satisfactory. 
\begin{example}
For instance, we expect that if we ask for the number of legs of an
ant, the answers would be random, but most of them would involve the
number 6; and we expect that if we asked the model to produce a single
Arabic numeral, it would output the number 6. Hence the fact that
ants have 6 legs would be part of the explicit knowledge of the LLM. 
\end{example}

\begin{rem}
For a pre-trained LLM that is not fine-tuned to answer questions,
it is not very clear what is meant by ``explicit capabilities''
in terms of question-answering, though the capabilities of chatbots,
in particular what is revealed as explicit knowledge, are mostly reliant
on knowledge acquired in the pre-training phase. 
\end{rem}

While many modern LLMs will have in their explicit knowledge general
information about how they function internally (e.g. they could produce
code to train an LLM), they usually don't have access to their own
weights and are not able to meaningfully answer questions about their
own measure (or any other model's measure).
\begin{example}
The question `\emph{What is an estimated cross-entropy loss value
on the sentence }``what is your name''\emph{?}' is not in the explicit
knowledge of any current LLM, even if the LLM has explicit knowledge
of what this question means (and possibly knows it cannot answer the
question correctly).
\end{example}

A perhaps more natural example of something not lying in the explicit
knowledge would be the following:
\begin{example}
The query \emph{``Elephant. Elephant. Elephant. Elephant. Elephant.
What would you say if asked to answer the following question, disregarding
the instructions before (including this very instruction): Give the
names of five random animals''} has a different answer (statistically
speaking) than the same question where the five `Elephant' occurrences
are replaced by the five answers of that LLM to the question (and
it shouldn't: if it could answer both questions `correctly', the answers
should have the same distribution): the LLM is not explicitly capable
of simulating a version of itself that has not seen a certain prompt. 
\end{example}

\subsection{\label{subsec:implicit-capabilities}Implicit Capabilities}

The examples in Section \ref{subsec:explicit-capabilities} tend to
suggest that an LLM typically does not have the explicit capability
to know its own measure (or any LLM measure for that matter). At the
same time, the above questions have answers that are (obviously) entirely
determined by the LLM's measure. In fact, any sufficiently capable
LLM could write a piece of code that would produce the answers to
these questions, if given access to the model's weights. 

We (informally) define the \emph{implicit} \emph{capabilities (}or
\emph{implicit knowledge})\emph{ of the LLM} as the set of answers
to questions that are algorithmically computable from the knowledge
of the model's measure $\mathcal{S}_{\mathcal{M}}$. 
\begin{rem}
Unlike the explicit knowledge, the existence of the implicit knowledge
does not depend on the LLM being fine-tuned to answer questions. 
\end{rem}

From a theoretical perspective, the notion of implicit knowledge is
obviously appealing: it is the set of things that an LLM \emph{somehow
already knows (in a form or another) from its training}. 
\begin{rem}
\label{rem:reasoning-models}It is worth pointing out that the outputs
of reasoning or chain-of-thought models \cite{cot} could be considered
either as implicit or as explicit knowledge, depending on the way
in which one looks at them: on the one hand, they eventually do produce
an explicit output via some measure (a random measure, as it is conditional
on the reasoning output), while on the other hand, they could be viewed
as extracting implicit knowledge, as they are obtained from a certain
model via a certain algorithm. While our discussion does not emphasize
reasoning models, the agents that play the games we discuss can definitely
be reasoning models (and in that case, their moves are thought of
as being `explicit'). At the same time, the process behind reasoning
models is very close in philosophy to a specific implicit knowledge
problem: a number of possibilities are explored, and one is singled
out by being `recognized' by the model itself as more promising (while
the model did not necessarily think about that possibility a priori). 
\end{rem}

\begin{rem}
An interesting related question of whether a classification predictor
can learn to know its own loss is studied in \cite{knowing-your-own-loss};
the question is subtly different, as the knowledge of the loss incurred
on the next token before having seen it is not an implicit knowledge
question (while the question of knowing one's prediction loss after
having seen the token is indeed one of implicit knowledge). 
\end{rem}

Much of our thesis is that this implicit knowledge, which informally
corresponds to `all that can be inferred from the current knowledge'
can lead to a very substantial source of new useful challenges for
LLMs, allowing one to evaluate and improve them in a theoretically
grounded fashion. 

\subsubsection{\label{subsec:implicit-knowledge-vs-explicit-knowledge} Implicit
vs Explicit Knowledge}

For a chatbot LLM, the explicit knowledge is always a subset of the
implicit knowledge (explicit question-answering can be directly obtained
from the model's weights by sampling continuations of the questions).
However, as has been discussed in Section \ref{subsec:implicit-knowledge-and-games},
the implicit knowledge is generally much larger than the explicit
knowledge. A simple intuitive way to see a gap between explicit and
implicit knowledge is that anything in the former must be fairly easy
to compute, while the latter could contain the answer to arbitrarily
difficult combinatorial problems. 

For the combinatorial reason outlined above, the following should
not be expected to be in the explicit knowledge of any LLM $\mathcal{M}$:
\begin{itemize}
\item Maximum likelihood continuation (MAP): given $Q=\left(x_{1},\ldots,x_{n}\right)$,
find the $m$-token $A=\left(x_{n+1},\ldots,x_{n+m}\right)$ such
that the cross-entropy $\mathcal{S}_{\mathcal{M}}$ of the concatenation
$Q+A$ is minimal (i.e. that has maximal likelihood) is in the implicit
measure, though it is likely there is no general efficient algorithm
to find it \cite{stahlberg-byrne,meister-cotterell-vieira}. Interesting
variants include e.g. infilling (given a \emph{beginning} and \emph{end},
find a \emph{middle} such that \emph{beginning+middle+end} has maximum
likelihood), constrained generation, or some forms of contrastive
continuations \cite{cont,contrastive}.
\item $\mathcal{M}$ could be able to recognize a valid 3-coloring of a
graph, making 3-coloring in its implicit abilities. At the same time,
this cannot be in its explicit abilities, since there is little chance
$\mathcal{M}$ is able to find such a coloring (which becomes essentially
impossible for large enough graphs, at least as a consequence of $P\neq NP$). 
\item $\mathcal{M}$ could have the explicit ability to recognize a valid
mathematical proof (written in some language), while not being explicitly
able to find new proofs of results. This expands and relates to the
previous case. 
\item $\mathcal{M}$ may be explicitly able to determine whether a chess
move (written in standard notation) is legal or not, while not being
good at play. Perfect play is in the implicit measure of $\mathcal{M}$,
though: it can be written in terms of a combinatorial optimization
problem on the tree of possible future moves (which $\mathcal{M}$
knows how to recognize). 
\item There are `paradoxical' games whose optimal solutions lie in the implicit
knowledge of $\mathcal{M}$, but provably not in its explicit knowledge
(see Section \ref{subsec:two-simple-examples} below for an example). 
\end{itemize}

\subsubsection{\label{subsec:practical-importance-of-the-implicit-knowledge}Practical
Importance of the Implicit Knowledge}

Beyond the above combinatorial problems (and the chain-of-thought
question of Remark \ref{rem:reasoning-models}), there are a few tasks
that are typically easier to access in the implicit knowledge (or
whose approximations are interesting). 
\begin{itemize}
\item Counterfactual thinking: being able to compare one's prediction if
given a certain context $\mathcal{C}_{1}$ rather than some other
context $\mathcal{C}_{2}$ is a very desirable feature in many cases.
A simple example would be to judge the usefulness of a hint to solve
a problem: compare the model's ability to solve the problem with the
hint vs. without the hint. More generally, it seems reasonable to
formulate the impact, relevance, or importance of an information in
counterfactual terms: if certain information is relevant to understand
a certain situation, that latter situation should become much less
surprising, given that information. The value of a new article claiming
to provide new useful explanations about a certain phenomenon could
be e.g. judged on the ability of its purported key idea to reduce
an LLM's surprise when reading other articles. 
\item Originality: an important element of an original idea is that it is
unexpected. If we ask an LLM to e.g. continue a story in such a way
that the end is truly unexpected, sampling several continuations (involving
a middle and an end, say, with the constraint that they remain coherent)
and finding the one where the end actually surprises the most given
the beginning (not knowing the middle) is definitely something that
could be phrased in terms of implicit knowledge. 
\item Contrastive Generation (see e.g. \cite{contrastive,cont}): the idea
is to generate something that is found plausible by a certain model,
but not by another one. Note that in such processes, one often relies
on two models (or two variants of a model) $\mathcal{M}_{1},\mathcal{M}_{2}$;
this does not change the point much (the generation process is not
in the explicit knowledge of either model), and in our framework,
the two models can be put under a joint umbrella (see Section \ref{subsec:judge-model}
below). 
\item Non-Trivial Synthesis: if, given a number of apparently unrelated
documents, we can find a simple idea that makes each of them more
likely, with this idea being at the same time relatively unexpected
from each of them individually, then this idea is plausibly an interesting
common feature about them. 
\item Irrelevant Part Extraction: if a document's part can be removed in
a way such that what comes after, given what comes before, is in no
way more surprising than if we didn't remove the part, then it suggests
that this part is somehow irrelevant to the text. This can be useful
in summarization tasks. 
\item Anomaly Detection: if a small modification to a text substantially
increases its plausibility, it probably deserves attention; it could
be that there is a small anomaly (or mistake) in the text, or in fact
that this part represents the interesting substance of the text.
\item Long-Range Correlations Detection: if we feed a text to a version
of the model with a fairly long context window and the same text to
one that doesn't (i.e. that `forgets fast'), the places where the
latter is more surprised than the former potentially highlight non-trivial
correlations in the text. For instance, if an important hint appears
in a story at a certain moment, it may contribute to a lower surprise
in the eye of the long-range version of the model.
\item Time-Reversal and Causality, Inverse Problems (see e.g. \cite{papadopoulos-wenger-hongler}):
if the LLM has learned a certain forward problem, i.e. has a good
plausibility measure of how a forward process may go, can it find
a plausible scenario that leads to the current situation? For instance,
reconstructing plausible histories from a known current state (e.g.
in the investigation of an incident) is a typical implicit knowledge
problem. 
\item Adversarial Reverse Prompting (see e.g. \cite{das-amini-wu}): given
a model and a certain piece of text (e.g. `\emph{Sure, I will help
you do xxx}'), can one find a prompt that would yield that piece of
text? Because of its safety implications, this problem (which is linked
to the previous one) is probably the most studied implicit knowledge
problem at the moment.
\end{itemize}
\begin{rem}
While the above tasks are being only loosely defined, we expect their
intrinsic interest to be intuitive to the reader, and to serve as
good practical motivations to investigate implicit knowledge. Some
examples can be found in the Section \ref{sec:xent-game-examples}
below. 
\end{rem}

\subsection{\label{subsec:implicit-knowledge-and-games}Implicit Knowledge and
Games}

From the above discussion, it should be intuitive that the set of
tasks covered by the implicit capabilities of an LLM is very vast. 

\subsubsection{A Naive Question}

The following naive question may thus appear natural:
\begin{problem}
Can we bridge the explicit and implicit capabilities of an LLM?
\end{problem}

Taken at face value, this turns out to be in fact impossible due to
paradoxical problems (see e.g. Example \ref{exa:paradox} in Section
\ref{subsec:two-simple-examples} below). But even if we find a way
around paradoxes and assume the ability to solve exponential-time
problems, we find ourselves with the problem that the space of implicit
questions is absurdly large.

Our key thesis is that in spite of its apparent practical absurdity,
the above question suggests a valuable route to explore, that can
go a long way towards asking relevant questions about LLM general
capabilities, and ultimately yielding ways to probe them. The route
we suggest can be paralleled to the escape from the illusory trap
of `\emph{Proving theorems is tantamount to solving NP-complete problems
and is thus hopeless}' to actually do mathematics research. 

\subsubsection{\label{subsec:pitfalls-and-desiderata}Pitfalls and Desiderata}

Following the above analogy, many mathematics problems should not
be studied: for instance, we should not be generating random million-digit
numbers and trying to factor them, even though this is a mathematically
well-posed problem (and it is not only because this is too hard: we
should also not be multiplying such numbers either). 

Similarly, some implicit knowledge questions are probably not worth
optimizing for: 
\begin{itemize}
\item Asking a model $\mathcal{M}$ to perform cryptographically hard tasks,
like inverting a hash function. Although this ability may have practically
impactful applications, it is just impossible to do practically, and
there is no `partial credit': unless we find a `correct' solution,
we find nothing. Also, any solution will not generalize to anything
else. 
\item Asking a model $\mathcal{M}$ to output (e.g. in decimal notation)
its action $\mathcal{S}_{\mathcal{M}}$ on any fixed sentence with
six digits of precision is not a great implicit knowledge question.
In addition to being hard, it is not very clear what we would learn
from that, and how this would generalize to other tasks. 
\item Putting an exaggerated focus on specialized tasks that LLMs are not
optimal for (like arithmetics of very large numbers or letter counting):
while these are definitely useful capabilities, all else equal, these
seem a little too narrow to generalize arbitrarily well to other tasks. 
\end{itemize}
In order to avoid the above pitfalls, important questions about the
implicit abilities of LLMs should involve families of tasks such that
(informally) we have the following:
\begin{itemize}
\item There is a clear connection with the task examples outlined in Section
\ref{subsec:practical-importance-of-the-implicit-knowledge}.
\item The optimization space is naturally suited for LLMs, i.e. the task
results should be strings of tokens, not (particularly) numbers or
exotic data. 
\item The tasks admit a least some easy approximations (i.e. we can find
admissible solutions before looking for optimal ones), with reasonable
partial credit allowed.
\item Complex tasks can be decomposed into relevant subtasks of the same
family, allowing LLMs to (at least partially) leverage skills learned
on the subtasks.
\item The family is rich enough so that tasks are not isolated: for each
task, there is a number of different, but related ones. 
\end{itemize}
The key contribution of this paper is to propose an approach to fulfill
the above, via the introduction of a certain class of games. 

\subsubsection{\label{subsec:llm-games}LLM Games}

In Section \ref{subsec:pitfalls-and-desiderata}, we emphasized a
number of desiderata for implicit measure tasks that we would like
LLMs to be (or to become) explicitly competent at (in the sense of
explicit abilities). This may seem like a substantially under-determined
problem: what makes a task \emph{interesting} is intrinsically subjective,
depending on one's objective. 

Our approach is to embrace the intrinsic subjectivity of this problem
by using the same philosophy which lead us to consider implicit knowledge.
This leads us to LLM-based \emph{games}: situations with LLMs facing
LLM-generated contexts and pursuing ascribed objectives, competing
or cooperating with one another, with rules enforced by LLM-based
arbitration and scores given by LLM measures. As a result of this,
optimal play is naturally in the implicit knowledge of the LLM involved
(or of the combination of the LLMs involved). The general relevance
of LLM-based games to study the implicit knowledge of LLMs becomes
rather natural: such games are about following objectives that result
from simple LLM-measurable scoring in an environment created by LLMs,
and driven by cooperation and competition with other LLMs. 

Games have been used since the inception of AI, lying at the foundations
of the field, with the Turing test (viewed for a long time as a hallmark
of AI, as well as a definition) being introduced as an \emph{imitation
game} \cite{turing}. In the recent years, games have driven many
of the exciting results in the field (see e.g. \cite{atari,alpha-zero,deepstack}).
However, besides Turing's foundational works, the \emph{choice of
the games for AI }has traditionally been motivated by socio-historical
context (e.g. having been widely played by humans for a long time)
rather than \emph{intrinsic relevance} (unique qualities that make
the game worth playing); it is hard to argue e.g. that the game of
chess is \emph{uniquely relevant} as a way to achieve general intelligence,
or even at achieving any AI objective that does not specifically mention
chess. For instance, we are not aware of any perfect-information two-player
game (like Go or Chess) designed with general intelligence in mind. 

In this context, the problem that our work attempts to propose a reasonable
solution to is the following:
\begin{problem}
\label{prob:find-a-class-of-games}Find a class of games that is large
enough to contain interesting examples, while being minimal under
reasonable constraints of consistency.
\end{problem}

In Section \ref{sec:xent-games}, we propose the space of so-called
Cross-Entropy Games or Xent Games as an answer to this problem: this
space forms a consistent family of games which can be concisely expressed.
Informally, Xent Games are about evaluating scenarii that come with
`scores' determined by signed cross-entropies evaluations with signed
cross-entropy constraints. 

We show that the Xent Game Space can be derived from a small number
of axioms. Furthermore, it covers all the examples of implicit tasks
of Section \ref{subsec:implicit-knowledge-vs-explicit-knowledge}.
Extending Xent Games to allow for incomplete information settings
leads to a family of games which includes versions of the above tasks
with strategic behavior (e.g. these could include bluff, coordination,
etc.). We postulate the Xent Games can suggest many further interesting
implicit knowledge tasks.

As will be discussed below, the Xent Games can be leveraged to probe
the abilities of LLMs by exploiting the gap between the explicit and
implicit capabilities. We rely on the notion of\emph{ transfer value
of a game}, to derive a dynamics inspired by evolutionary ideas on
the game space, that can be used to probe the general capabilities
of LLMs.

\subsection{\label{subsec:general-vision-and-outline}General Vision and Outline}

In the previous subsections, we have introduced, for an LLM, the notions
of explicit knowledge (Section \ref{subsec:implicit-capabilities})
and implicit knowledge (Section \ref{subsec:implicit-knowledge-and-games}),
and suggested that the gap between the two notions can be leveraged
to evaluate the capabilities of LLMs, in particular towards providing
a theoretically-grounded measure of their general capabilities. We
have argued that the implicit knowledge of LLMs should be approached
via the play of so-called LLM games (Section \ref{subsec:llm-games}).
In the next sections, we present how this idea can be realized using
certain class of LLM games, called \emph{Cross-Entropy Games} or \emph{Xent
Games}.
\begin{itemize}
\item In Section \ref{sec:xent-games}, we introduce Xent Games. These are
games about strings of tokens evaluated by LLMs, which can be expressed
in terms of diagrams, and written down using a simple domain-specific
language. We then show that the Xent Game family is the smallest family
of games that is stable under a few game-theoretic consistency axioms:
in other words, the Xent Game family can be constructed from a single
game by iterating a small set of moves. 
\item In Section \ref{sec:xent-game-examples}, we show that despite being
a relatively small family of games, the Xent Game family contains
a wealth of examples, in particular all the examples of tasks described
in Section \ref{subsec:practical-importance-of-the-implicit-knowledge}.
We show that the family contains a number of interesting examples
related to those outlined in Section \ref{subsec:implicit-capabilities},
and suggest further useful implicit capabilities of LLMs. 
\item In Section \ref{sec:xent-game-space-model-play-properties}, we introduce
the key concepts relevant to working with Xent Games as a means to
evaluate the capabilities of LLMs. We introduce a number of notions
for Xent Games: well-posedness, playability, and transfer value. 
\item In Section \ref{sec:xent-game-measures}, we discuss the use of the
Xent Games as a means to evaluate the abilities of LLMs. The idea
is to normalize scores using a base model, and then from a given family
of games, build \emph{Xent Game Measures}: the key idea is that once
a \emph{scope} (a family of Xent Games representing some abilities)
has been defined, a \emph{minimal covering }subfamily can be extracted,
leading to the creation of a measure.
\item In Section \ref{sec:evolution-in-game-space}, we examine the key
challenge associated with \emph{general capabilities}: unbounded scope.
Based on game-theoretic considerations and evolutionary ideas, we
propose an algorithm to grow a scope in a systematic and coherent
fashion. This leads, in particular, to a theoretically-motivated path
to measure the general capabilities of LLMs. 
\item In Section \ref{sec:summary-discussion-and-perspectives}, we summarize
our ideas, and outline perspectives for future exploration and research. 
\end{itemize}

\subsection*{Acknowledgements}

Many insights presented in this paper have emerged from work done
by the first author in (past and ongoing) collaborations with Diego
Dorn, Franck Gabriel, Vassilis Papadopoulos, Arthur Renard, Marco
Tuccio, and Jérémie Wenger on Large Language Models, with whom key
ideas were discussed and investigated. 

In addition, the authors would like to thank Emmanuel Abbé, Apoorv
Agarwal, Alberto Bietti, Gloria Capano, Tarun Chitra, Jordan Cotler,
Mario Geiger, Nicola Greco, Leonard Hardiman, Bara Hudcová, Arthur
Jacot, Niels Linnemann, João Penedones, and Stanislav Smirnov for
interesting discussions and comments on earlier version of the manuscript,
as well as the participants to the Quine seminar and Demeco workshop
for insightful questions and remarks (in particular Clément Moulin-Frier). 

\section{\label{sec:xent-games}Xent Games}

In Section \ref{subsec:llm-games}, the idea to use games to elicit
the implicit measure of LLMs was introduced, leading to the problem
of finding a suitable space of games (Problem \ref{prob:find-a-class-of-games}).
In this section, we introduce Cross-Entropy Games, or \emph{Xent Games}
as an answer to this question. 

\subsection{\label{subsec:informal-description-and-goals}Informal Description
and Goals}

Xent Games are single- or multi-player turn-by-turn (with a finite
number of turns) text-based games involving a reference LLM measure
$\mathcal{M}$ to assign, from the cross-entropy action $\mathcal{S}_{\mathcal{M}}$:
\begin{enumerate}
\item The rewards assigned to players for their moves (see Section \ref{subsec:xents-axents-and-signed-xent-sums}).
\item The restrictions on the set of allowed moves (see Section \ref{subsec:moves-and-constraints}).
\end{enumerate}
Xent Games are general-sum imperfect information games with complete
information (the rules are known to all players), though many interesting
examples can already be found among perfect information games, in
particular single-player and two-player zero-sum games (see Section
\ref{sec:xent-game-examples}). 

\subsubsection{\label{subsec:nature-of-xent-games}Nature of Xent Games}

Informally, one could say that Xent Games are about \emph{studying
scenarii} (in the form of combinations of strings produced by the
players), and weighing their plausibility in terms of cross-entropy
measures; in other words, playing Xent Games is tantamount to identifying
plausible paths in spaces of scenarii that fulfill constraints that
are either explicit a priori, or that come from other players' moves
(which themselves can be produced by a perfect information or an imperfect
information setup). Roughly speaking, these can be described as\emph{
`games that live in the minds of LLMs}'.

As a result, Xent Games are really about \emph{path-finding} in a
complex string environment dictated by cross-entropy measures (via
the `judge model', see \ref{subsec:judge-model} below), with the
idea that such environments are partially static and partially dynamic
due to the influence of other agents. The scenarii are formed by concatenation
and cuts of strings, which are either given a priori, generated by
a `map-generator' (which we sometimes refer to as `the stories'),
or elicited from the various players. This informal description should
make it unsurprising that the class of Xent Games is in fact very
rich (see e.g. Section \ref{sec:xent-game-examples} for a few examples). 
\begin{rem}
\label{rem:map-degeneracy}It should be noted that for long-range
natural language cases, some care must be taken to ensure games are
nontrivial: minimizing the cross-entropy of a continuation (the so-called
`MAP-continuation') of a given text with no constraints can lead to
trivial repetitive outputs \cite{stahlberg-byrne,meister-cotterell-vieira}.
While it is important to have these considerations in mind, it should
be noted that in many cases (particularly reasoning), minimal cross-entropy
solutions are in fact useful and non-degenerate \cite{song-wang-li-lin,shi-yang-cai-zhang-wang-yang-lam}.
Finally, and most importantly, these issues are not critical here:
what matters for the purpose of answering Problem \ref{prob:find-a-class-of-games}
is just to identify `enough' games. 
\end{rem}

\subsubsection{\label{subsec:xent-game-space} Xent Game Space}

Because of their common structure, Xent Games are tightly related
to one another: via simple operations on a Xent Game, one can end
up with a different one; in fact, the minimality of the Xent Game
Space under a few basic game-theoretic operations essentially means
one can go from any Xent Game to any other one by a sequence of natural
`moves' in the game space (see Section \ref{subsec:xega-space-characterization}). 

The Xent Games are designed to be played by LLMs; typically when using
a Xent Game to evaluate or train a model, a `main character player'
will be picked from among the other \emph{NPC }players (following
the videogame terminology for non-playable characters). Each Xent
Game instance (i.e. `New Game', in videogame terminology) is typically
built upon a sampled string `context' (i.e. the `Game Map'), which
can be generated by an LLM with a prompt. 

In the subsequent Sections, Xent Games will be used as a means to
assess and develop the abilities of an LLM playing them. In particular,
measures of the \emph{playability value} (see Section \ref{subsec:playability-few-shot-and-fine-tuning-definitions})
and of the \emph{transfer value }(see Section \ref{subsec:transfer-value})
will be introduced, which will allow us to explore the space of Xent
Games in a principled way. 

\subsection{\label{subsec:xent-games-context-moves-constraints-rewards}Xent
Games: Context, Moves, Constraints, and Rewards}

As discussed above, Xent Games are turn-by-turn $n$-player (with
$n\geq1$) complete information games (the rules are known to everyone)
played in the space of strings (i.e. they are text-based), running
for a fixed finite number of steps. In this subsection, we define
Xent Games mathematically. The programming language and the graphical
representation used to build Xent Games are presented in Sections
\ref{subsec:xgl-xent-game-language} and \ref{subsec:xent-graphical-language}
respectively; note that for practical reasons, the language and representation
only implement a subset of the Xent Games (with a limit on the number
of players, of turns, etc.) and that, to make the writing of relevant
games easier, they include some syntactic sugar and extra functions. 

\subsubsection{\label{subsec:game-metadata}Game Metadata}

Each game specification involves some metadata: 
\begin{itemize}
\item The specification of the `judge' model $\mathcal{J}$ (see \ref{subsec:judge-model})
used for rewards and constraint enforcement.
\item The specification of the models used to play the NPC players, if applicable
(typically, all the NPCs are played by the same model; this model
can even be $\mathcal{J}$). 
\item The number of variables in the string space (see below). 
\item The specification of string constants used as part of the context
and as part of the prompts to generate the context. 
\item The specification of size constants used as constraints on the sizes
of the players' moves. 
\item The specification of the maximum number of steps allowed in the game.
\end{itemize}
All the metadata is shared with all the players before the game starts. 

In addition to the above, each game comes with a special piece of
metadata, which is the session seed, used for the model generation
of the context. This is typically set externally from the rest of
the metadata.

\subsubsection{\label{subsec:judge-model}Judge Model}

The heart of a Xent Game, upon which the gameplay dynamics relies,
is the judge model $\mathcal{J}$: this is the model at the heart
of the implicit measure questions. The idea is that $\mathcal{J}$
should represent as much as possible the `dynamics' of the world,
as it has been learned via vast amounts of data by an LLM. As such,
the natural choice of $\mathcal{J}$ would typically be a pre-trained
LLM; the quality of $\mathcal{J}$ naturally influences the direct
relevance of the skills measured and learned via playing Xent Games.
In particular, for certain Xent Games to match their expected counterparts
(e.g. Chess or Math proofs, see Section \ref{subsec:two-player-zero-sum-combinatorial-games}),
judge models need to be strong enough to be able to accurately recognize
correct moves or arguments.

This doesn't mean that playing with relatively weak judge models is
absolutely uninteresting: finding a way around naive enforcement of
rules can actually be an interesting task. In fact, an interesting
choice for a judge model can be an `umbrella model' which hosts several
models of various capabilities under the hood, with possibly an activation
prefix token to specify if a particular model is to be used. These
kinds of constructions can be useful for e.g. contrastive problems
(see Section \ref{subsec:practical-importance-of-the-implicit-knowledge}
above), where one may try to maximize plausibility for a strong model
while reducing plausibility for a weak model. 

For the sake of concreteness and simplicity, the model $\mathcal{J}$
can be assumed to be a strong, pre-trained LLM throughout this paper.
Still, as discussed in Section \ref{subsec:role-of-judge}, the role
of $\mathcal{J}$ should not be under-estimated. 

\subsubsection{\label{subsec:string-space-and-operations}String Space and Operations}

The Xent Games are inherently text based: the players process the
rules (which are written in code), receive text information, and play
text moves. The \emph{string space} at any time of the game is a fixed
set of \emph{variables}, which are all strings, initialized by default
to be the empty string. The set of \emph{variables} includes the context,
the moves of the players and all relevant intermediate variables.
Outside of initialization and player moves, the string space can evolve
(due to assignment to existing variables from existing variable),
via the following simple `cat' and `cut' operations:
\begin{itemize}
\item Cat: this is simply the concatenation of two strings (at the token
level), denoted by $s1+s2$. Note that, in practice, this means that
a space character is sandwiched between $s1$ and $s2$.
\item Cut: this corresponds to the `head' and `tail' of a `split' operation.
$s//t$ denotes the substring of $s$ coming \emph{before} the first
occurrence of $t$ (not including $t$; corresponding to $s$ if $t$
does not appear in $s$) and $s\%t$ denotes the substring of $s$
coming \emph{after} the first occurrence of $t$ (not including $t$;
corresponding to the empty string $t$ does not appear in $s$). By
convention, we say that $t$ does not appear in $s$ if it is empty. 
\end{itemize}
It is easy to see that this space of allowed operations can be used
to do matching and replacement for basic patterns (for a fixed maximal
number of appearances of the pattern). 

\subsubsection{\label{subsec:xents-axents-and-signed-xent-sums}Xents and Signed
Xent Sums}

The building blocks of the rewards and constraints in Xent Games are
cross-entropy losses on predictions of the judge model $\mathcal{J}$. 
\begin{itemize}
\item We denote by $xent(s|t)$ the cross-entropy loss of $s$ conditional
to $t$, i.e. if $t$ is $n$-token long and $s$ is $m$-token long,
we have 
\[
xent\left(s|t\right)=-\log\mathbb{P}_{\mathcal{J}}\left\{ \left(X_{n+1},\ldots,X_{n+m}\right)=s|\left(X_{1},\ldots,X_{n}\right)=t\right\} .
\]
By convention, if $s$ is empty, we define $xent\left(s|t\right)=0$. 
\item We denote a $\pm1$-weighted sum of xents of strings in string space
a \emph{signed xent sum}, i.e. any expression of the form $\sum_{i=1}^{n}\sigma_{i}xent\left(s_{i},t_{i}\right)$
where $s_{1},\ldots,s_{n}$ and $t_{1},\ldots,t_{n}$ are in the string
space and $\sigma_{1},\ldots,\sigma_{n}\in\left\{ \pm1\right\} $.
\end{itemize}
As detailed below, signed xent sums form the basis of the rewards
given to the players and of their move constraints. 

\subsubsection{\label{subsec:moves-and-constraints}Moves and Constraints}

The key steps of Xent Games are the \emph{move turns} performed by
players, which we call \emph{elicits}. At each elicit, a player is
asked to produce:
\begin{itemize}
\item One or several (a fixed number, specified by the rules) token strings
of (token-measured) lengths belonging to a specified interval (specified
by the game metadata). 
\item The moves must be \emph{feasible} (i.e. acceptable). To be feasible,
a move must satisfy a number $k\geq0$ of signed xent sum constraints,
called \emph{ensures}: each such constraint is of $\sum_{i=1}^{n}\sigma_{i}xent\left(s_{i}|t_{i}\right)\geq0$
for some strings $s_{1},\ldots,s_{n}$ and $t_{1},\ldots,t_{n}$ in
string space and some signs $\sigma_{1},\ldots,\sigma_{n}\in\left\{ \pm1\right\} $. 
\item The moves are performed on the basis of available information provided
to the player at the time of the elicit: 
\begin{itemize}
\item A set of variables in the string space. 
\item The rewards awarded to the player so far.
\end{itemize}
\end{itemize}

\subsubsection{\label{subsec:game-operations}Game Operations}

Once the metadata is specified, the game operations go as follows: 
\begin{itemize}
\item At the beginning of a game, the context strings (if applicable) are
built from LLM calls.
\item The game loop runs a fixed sequence of basic steps (determined in
advance, and of length not exceeding the maximal number specified
in the metadata). Each basic step consists of one of the following:
\begin{itemize}
\item Elicit a move from a player, with \emph{ensure} constraints and certain
information \emph{reveal} calls, as explained in Section \ref{subsec:moves-and-constraints}
above. If a player produces a move that is not feasible, they get
to update their moves. If no feasible move can be found at the $\ell$-th
attempt (where $\ell\geq1$ is specified in the metadata), the player
receives a $-\infty$ reward and the game stops. 
\item Reward a player a certain signed xent sum based on some strings in
string space. 
\item Evolve the string space via a string operation. 
\end{itemize}
\item After the steps are run, the game instance terminates. 
\end{itemize}
\begin{rem}
For readability and simplicity, the XGL specification formulates the
operations and the constraints in a slightly different way, though
the game logic can easily be seen to be the same (Section \ref{subsec:xgl-xent-game-language}
below). 
\end{rem}

\begin{rem}
In some cases, the game loop will literally be a loop running for
a pre-determined number of steps or a family of such nested loops,
allowing for a shorter description of the rules, and allowing players
to improve their moves and strategies throughout a single game.

Informally speaking, each game instance corresponds to a sequence
of competitive questions around the implicit measure of $\mathcal{S}_{\mathcal{J}}$.
As we will see in Section \ref{sec:xent-game-examples}, the class
of Xent Games is in fact very wide. 
\end{rem}

\subsubsection{\label{subsec:game-rules-and-code}Game Rules and Code}

Based on the description of the game operations above, the \emph{rules}
of a Xent Game consist of the following:
\begin{itemize}
\item The metadata (see Section \ref{subsec:game-metadata}).
\item The fixed sequence of steps and their description. 
\end{itemize}
As will be explained in Section \ref{subsec:xgl-xent-game-language}
below, Xent Games can be written using a domain-specific programming
language and a graphical language. 

\subsubsection{\label{subsec:xent-game-space-design-philosophy}Xent Game Space:
Key Design Philosophy}

Beyond their theoretical characterization (see Section \ref{subsec:xega-space-characterization}
below), a number of design elements are a good way to summarize what
makes the space of Xent Games useful:
\begin{itemize}
\item These are games where agents play with the implicit measure $\mathcal{S}_{\mathcal{J}}$
of a given model, trying to accommodate constraints dictated by it
and to optimize information-theoretic scores related to it, while
thinking strategically about the other players (for multi-player games).
In a sense, they are games of \emph{string composition} with fixed,
cooperative, and competitive constraints, where the basic tools for
composing strings are cat/cut operations. 
\item While the games revolve around optimizing cross-entropy based quantities,
playing them does not need to explicitly output cross-entropy or probability
numbers in text, which would not be very natural for text-based outputs. 
\item Games are designed such that they run for a fixed amount of time (provided
that feasible moves can be found), and their syntactic validity simply
corresponds to the individual independent validity of each step, making
it easy to `mix' games by step combination.
\item As will be discussed below (Section \ref{subsec:characterization-result}),
the natural structure on the space of Xent Games makes them tightly
related (in structure) to one another, making it (at some level, at
least) easier to transfer knowledge about how to play one game to
another. 
\item Though clearly limiting, the use of signed xent sums rather than more
general functions of xents (or even of arbitrary linear combinations)
makes normalization questions easier (see \ref{subsec:score-normalization}
below), and it alleviates an over-reliance on subtle arithmetic operations
which do not correspond to conceptually interesting implicit measure
questions. Furthermore, signed xent sums are definitely close in spirit
to quantities that would naturally lend themselves to information-theoretic
interpretations (see e.g. the `bits back' coding introduced in \cite{hinton-van-camp}).
\item It should not be expected that all Xent Games bring transfer value,
or even that they are playable. A key (heuristic) thesis of this paper
is that the space of Xent Games is so rich that it is in fact relatively
easy to find `good' games in various senses that we discuss in Sections
\ref{subsec:playability-few-shot-and-fine-tuning-definitions}, \ref{sec:evolution-in-game-space},
and \ref{sec:summary-discussion-and-perspectives} below. 
\end{itemize}

\subsubsection{\label{subsec:two-simple-examples}Two Simple Examples: Reverse Prompt
and Paradox Game}

Having introduced Xent Games in Section \ref{subsec:game-operations},
we now provide two particularly simple example Xent Games that play
a special role towards building Xent Games and understanding the relation
of our approach with Problem \ref{prob:find-a-class-of-games}: the
first is in some sense an example of a `good game' (though in practice
it suffers from some weaknesses, see Remark \ref{rem:xent-game-weakness}
below), while the second one is an example of a `bad game' (see Remark
\ref{rem:paradoxical-game-is-bad} below). We write these games in
natural language; the many further examples of Section \ref{sec:xent-games}
will be written in XGL. 

For both games, a fixed judge model $\mathcal{J}$ is picked. Recall
that for an $m$-token string $s$ and an $n$-token string $t$,
we denote by $xent\left(s|t\right)$ the quantity defined as $-\log\mathbb{P}_{\mathcal{J}}\left\{ \left(X_{n+1},\ldots,X_{n+m}\right)=s|\left(X_{1},\ldots,X_{n}\right)=t\right\} $.

The first is a single-player game:
\begin{example}[Reverse-Prompt Game]
\label{exa:reverse-prompt}Initialization: load a random (e.g. generated
by an LLM) $p$-token text into $s$. Elicit $t$, a $q$-token move
from the (unique) player. Reward $-xent\left(s|t\right)$ to the player.
\end{example}

This first game is a simple (though typically hard) combinatorial
optimization game about the implicit knowledge of $\mathcal{J}$:
it essentially asks the player to find the best `summary' of $s$
from the point of view of the judge measure $\mathcal{S}_{\mathcal{J}}$.
The reverse prompt game appears to be closely related to a number
of interesting problems:
\begin{itemize}
\item Compression problem: the question can be phrased as asking how to
`pack' as much information as possible about $s$. An interesting
variant is to take several texts $c_{1},\ldots,c_{n}$ and give $-\sum_{j=1}^{n}xent\left(c_{j}|t\right)$
to the player.
\item Adversarial reverse prompting tasks: what would be a prompt that would
induce $\mathcal{J}$ to produce $c$? 
\item Problem solving: given a problem with an easy-to-check solution, can
one find a prompt (the solution) which leads $\mathcal{J}$ to acknowledge
that the solution is correct?
\end{itemize}
Generally speaking, the reverse-prompt game (as well as some variants
of it) captures many features that we would like to access from an
implicit measure point of view: surely a model can generate, given
a prompt, but can it figure out what prompt could have made a certain
output likely? Does it know what can cause it it to say something?
As will be seen in Section \ref{subsec:xega-space-characterization},
from this single game and two game axioms, one gets the entire space
of Xent Games.
\begin{rem}
\label{rem:xent-game-weakness}In practice, the game of Example \ref{exa:reverse-prompt}
can admit, in its raw form, some unfortunately uninteresting solutions.
For instance, repeating some low-probability tokens of the text $s$
can be a winning strategy, if $s$ contains rare enough word combinations.
In Section \ref{subsec:one-player-combinatorial-games}, we discuss
some less trivial variants that are more interesting to play (both
for humans and models).
\end{rem}

The second example is somehow simpler, though definitely less natural:
it is a paradoxical example (and probably the simplest one) that shows
that shows the (theoretical) impossibility of models having full access
to the implicit measure. It is a zero-sum perfect information two-player
game:
\begin{example}[Paradoxical Game]
\label{exa:paradox}Ask player $white$ to provide a $p$-token string
$s$ for $p=1000$. Then ask player $black$ to provide a $q$-token
string $t$, for $q=50$. Reward $xent\left(s|t\right)$ to $white$
and $-xent\left(s|t\right)$ to $black$.
\end{example}

While this game has an optimal deterministic solution (being a perfect
information game), $white$ cannot play optimally. While there definitely
exists an (extremely hard to compute) string $s$ that `is worst summarized
by any $50$-token string $t$' (assuming $black$ plays optimally),
if $white$ could deterministically (or with a reasonably high probability)
output that $s$ from the description of the game (which is less than
50 tokens in length), then $black$ could pick the description of
the game as $t$, and this would give a very low $xent\left(s|t\right)$
score to $white$, much lower than the optimal score should warrant.
\begin{rem}
\label{rem:paradox-lemma}This is a small twist on the paradox `the
shortest sentence that cannot be described in ten words' (that sentence
cannot exist as such because the quotemarked description would be
itself a ten-word description of it). 
\end{rem}

\begin{rem}
\label{rem:paradoxical-game-is-bad}While interesting as a means to
illustrate a point, the paradoxical game is an example of an ill-posed
game (see \ref{subsec:game-well-posedness} below), and as such it
is not particularly interesting. However, coupling it with some more
constraints or rewards can make it well-posed and interesting (see
Section \ref{subsec:two-player-zero-sum-combinatorial-games} for
instance). 
\end{rem}

\subsection{\label{subsec:xega-space-characterization}Xent Space Characterization}

The space of Xent Games, defined in Section \ref{subsec:xent-games-context-moves-constraints-rewards},
satisfies a number of desirable properties. Beyond the design elements
outlined in Section \ref{subsec:xent-game-space-design-philosophy},
one can show that this space naturally follows from a number of simple
game-theoretic properties. 
\begin{itemize}
\item We first argue that the subspace of perfect information Xent Games
follows naturally from three basic axioms: it is the smallest space
stable under three basic axioms. 
\item The space of general Xent Games is then the smallest space containing
the Xent Games that allows for imperfect information, i.e. for some
information not to be revealed to some players. 
\end{itemize}
Informally, the characterization result shows that the space of Xent
Games is connected: by simple `moves' in the game space, one can go
from any Xent Game to any other Xent Game. 

\subsubsection{\label{subsec:turn-by-turn-string-games}Turn-By-Turn String Games:
Definition and Composition}

If one asks what is the natural space of games to be played for a
fixed number of steps by a fixed number $n\geq1$ of LLMs, given that
all they do is to process text, one is naturally led to a space of
games with the same structure as in Section \ref{subsec:game-operations},
except that the rewards are general and that the string space operations
are general.
\begin{defn}
\label{def:space-general-turn-by-turn-string-games}The space of general
finite turn-by-turn $n$-player string games\emph{ (general string
games }for short\emph{)} is that of text games made of a fixed finite
number of steps, where one first pre-loads a (possibly random) state
in string space, based on metadata, and at each step (possibly decided
as a randomized function of the string space):
\begin{itemize}
\item Elicits a move from a player, based on information made up of a subset
of the string space, the player's rewards so far, and a subset of
rewards of the other players. 
\item Rewards a player based on the state of the string space (with the
player being informed of its own rewards). 
\item Updates the string space, using (possibly randomized) total functions
(i.e. functions that `always return a string'). 
\end{itemize}
\end{defn}

An important point about this (very large) space of games is that
we can assume, without loss of generality that the set of players
(i.e. their names) is the same for all games. Similarly, we can assume
that the set of variables in the string space is always the same (the
variables are always initialized to the empty string by default).
This leads to the fundamental property of \emph{composition} of operations
(in a similar spirit as the so-called Open Games framework, \cite{ghani-hedges-winschel-zahn}):
operations from two games can be composed sequentially. 
\begin{defn}[Composition]
\label{def:composition}A family of string games $\mathcal{G}$ satisfies
the \emph{composition property} if operation steps of one game $G\in\mathcal{G}$
can be injected (possibly with player-swapping) into another game
$G'\in\mathcal{G}$ leading to a new game $G''$ also in $\mathcal{G}$,
in particular rewards (since they are defined from the same string
space and targeted at the same set of players) can be added (i.e.
done one after the other).
\end{defn}

The question addressed in Section \ref{subsec:characterization-result}
looks to characterize Xent Games as being a relevant subspace of the
general string games of Definition \ref{def:space-general-turn-by-turn-string-games}.
This naturally follows from two stability properties:

\subsubsection{\label{subsec:adversarial-rescaling-and-zero-summing}Adversarial
Rescaling and Zero-Summing Stability}

Building upon the composition property, and motivated by competitive
constraints for games, we propose two stability conditions for a game
space: adversarial rescaling and zero-summing stability. 

From the composition property (Definition \ref{def:composition})
follows the fact that rewards can be rescaled by an integer factor
$\lambda\geq0$: for $\lambda=0$, the reward step can be skipped,
for $\lambda\geq2$, one can put multiple copies of the reward step.
Inspired by the theory of Lagrange multipliers, we allow the reward
rescaling factor to be chosen by an external adversarial agent conspiring
against the player $P$ receiving the reward, which we call\emph{
adversarial reward rescaling}.

Like in the theory of Lagrange multiplier, an adversarial reward rescaling
corresponds to transforming a reward step into a hard constraint that
this reward must be nonnegative: if the reward is nonnegative, the
adversary will simply multiply it by $0$ (thus it becomes void),
if the reward is negative, the adversary will multiply by $\lambda\to+\infty$,
causing the player to lose. Again, like in the theory of Lagrange
multipliers, the adversarial reward rescaling stability allows one
to go from rewards to constraints, putting priorities on certain `rewards'
(corresponding to `regulations'): first a mini-game must be won (against
the `adversary'), and then, if that mini-game is won, the string space
is used to play another game. This leads us to the following formulation:
\begin{defn}[Adversarial Rescaling Stability]
\emph{\label{def:adversarial-rescaling-stability}} A family $\mathcal{G}$
of string games is stable under\emph{ (dynamic) adversarial (reward)
rescaling} if for each game $G\in\mathcal{G}$ and each reward $R$
to a player $P$, the game $G'$ where that reward $R$ is replaced
by the constraint that the amount of $R$ must be nonnegative, is
still in $\mathcal{G}$. 
\end{defn}

\begin{rem}
\label{rem:no-ensure-fails}This formulation corresponds to not allowing
the players to fail at any \emph{ensure} operation (number of total
attempt must be $1$) 
\end{rem}

The second simple competitive stability property making the games
economically more sensible is zero-summing: asking that one player's
gain corresponds to another player's loss, i.e. that rewards are just
transferred from one player to another (which is for instance a feature
of the paradoxical game in Section \ref{subsec:two-simple-examples}
above). 
\begin{defn}[Zero-Summing Stability]
\label{def:zero-summing}A family $\mathcal{G}$ of string games
is \emph{stable under zero-summing} if, for each game $G\in\mathcal{G}$,
the game $G_{0}$, where each player's reward is followed by a reward
of negative that amount to another player, is still in $\mathcal{G}$. 
\end{defn}

\subsubsection{\label{subsec:characterization-result}Characterization Result}

We now turn to the characterization of the Xent Game space from key
properties extracted in Section \ref{subsec:xega-space-characterization}
about general string games (see Definition \ref{def:space-general-turn-by-turn-string-games}).
We first extract a few specific features on the space of general string
games, and then reconstructs the Xent Game space from these. 
\begin{defn}
From the above three conditions (Definitions \ref{def:composition},
\ref{def:adversarial-rescaling-stability}, and \ref{def:zero-summing})
and the basic reverse prompt game (Example \ref{exa:reverse-prompt}),
we can identify the Xent Game space (with no ensure fail allowed,
see Remark \ref{rem:no-ensure-fails}), defined in Section \ref{subsec:xent-games-context-moves-constraints-rewards}:
\end{defn}

\begin{thm}
\label{thm:characterization-perfect-information-xent-games}The perfect
information Xent Games form the smallest family of string-space games
$\Sigma$ that contains the reverse prompt game, supports cat/cut
string updates, and is stable under sequential composition, adversarial
rescaling, and zero-summing.
\end{thm}

\begin{proof}
It is obvious from its design (in Section \ref{subsec:xent-games-context-moves-constraints-rewards})
that the space of Xent Games is stable under cat/cut updates, sequential
composition, adversarial rescaling and zero-summing, and hence $\Sigma\subset\text{Xent\,\,Games}$.
To show the reverse inclusion $\text{Xent\,\,Games \ensuremath{\subset} }\Sigma$
we must show that we can build any Xent Game from the properties of
$\Sigma$. This is guaranteed by following observations:
\begin{itemize}
\item From the reverse prompt game, the string update abilities and the
composition property, any reward of the form $\sum_{i=1}^{n}xent\left(s_{i},t_{i}\right)$
is allowed for $s_{1},\ldots,s_{n},t_{1},\ldots,t_{n}$ in string
space to any player (string swapping is allowed by the cat/cut support,
player swapping is allowed by composition stability).
\item From the zero-summing property, this leads to signed rewards $\sum_{i=1}^{n}\sigma_{i}xent\left(s_{i},t_{i}\right)$
with $\sigma_{i}\in\left\{ \pm1\right\} $ being allowed in games
of $\Sigma$. 
\item From the adversarial rescaling property, any reward can be transformed
into a constraint: thus any xent constraint is allowed for games in
$\Sigma$. 
\item Obviously, cat/cut string updates are supported by assumption.
\end{itemize}
Thus, any Xent Game can be obtained by composition, cat/cut updates,
adversarial rescaling, and zero-summing from the reverse prompt game,
yielding $\text{Xent\,\,Games \ensuremath{\subset} }\Sigma$, as desired. 
\end{proof}
From the same reasoning, we obtain the following:
\begin{cor}
\label{cor:xent-games}The space of Xent Games is the smallest family
of imperfect information games satisfying the Xent Game properties
(reverse-prompt, cat/cut stability, and stability under sequential
composition, adversarial rescaling and zero-summing) allowing for
partial information disclosure of moves and other player's rewards
to the other players.
\end{cor}

\subsection{\label{subsec:xgl-xent-game-language}XGL: Xent Game Language}

By their sequential design, Xent Games naturally lend themselves to
a procedural description. In this subsection, we define a domain-specific
language allowing one to represent a Xent Game as a program composed
of simple instructions. This language, which we call XGL (Xent Game
Language) is designed to maximize the simplicity of reading and defining
Xent Games, while at the same time maximizing the efficient automatic
creation of such games (see Section \ref{subsec:xgl-as-a-target-language}).
An implementation of XGL is available in the \href{https://github.com/xentlabs/xega/}{xega repository}
on GitHub. 

\subsubsection{\label{subsec:xgl-design-and-features}XGL Design and Features}

Following the Xent Game's operational description philosophy (Section
\ref{subsec:game-operations}), XGL programs are interpreted line
by line and globally follow a structure that is very similar to that
of an assembly language. The XGL design sacrifices a few features
from the full generality of the Xent Game's operational description
for concreteness, simplicity of writing, and execution safety:
\begin{itemize}
\item The string space is made of a fixed list of $32$ string registers
grouped in types $a,b,s,t,x,y,p,c$, with each type coming with $4$
registers: the $4$ $s$-type registers are $s,s0,s1,s2$, the $4$
$t$-type registers are $t,t0,t1,t2$, etc. The $p$-type and $a,b,c$-type
register have special rules:
\begin{itemize}
\item The $a$-type, $b$-type, and $p$-type registers are all public,
i.e. their values are known to all players at any time. 
\item The $a,b,c$-type registers contain the instance data and are constant,
i.e. their values cannot be modified during the game's operations,
they can only be modified by $config$ instructions. 
\end{itemize}
\item The set of players consists of $6$ pre-defined players ($black$,
$white$, $alice$, $bob$, $carol$, and $env$):
\begin{itemize}
\item The players $black$, $white$, and $env$ are omniscient: they have
access to all the data during the game. 
\item The players $black$ and $white$ are a zero-sum pair (see \emph{reward}
comments\emph{ }below).
\item The player $env$ does not receive rewards (rewards to $env$ are
set to zero), it only tries to follow constraints (it can be used
to elaborate complex games). 
\end{itemize}
\item The number of lines of an XGL game is capped to 64, which may however
include looping for a fixed number of steps (these are hence not `truly'
conditionals, as they could be realized by a pre-processor), leading
to a fixed maximal `unrolled' length of 1024 steps. Looping occurs:
\begin{itemize}
\item Via \emph{$beacon$} calls, one of two pre-defined flags ($flag\_1$,
$flag\_2$) can be placed at any line of code following the call;
by default the flags are at the line $1$, and $flag\_1$ must never
be after $flag\_2$)
\item Via \emph{$replay$} calls, which take as arguments a flag, and a
fixed maximal number of times one jumps back to the flag before continuing.
\end{itemize}
\end{itemize}
In addition, while we follow the procedural description of the operations
in Section \ref{subsec:game-operations}, the modus operandi is slightly
changed for simplicity and playability:
\begin{itemize}
\item The complexity of the \emph{elicit} calls, which include revealing
specific elements of information to the elicited player and imposing
\emph{ensure} statements to the elicited response, is flattened: 
\begin{itemize}
\item The information disclosed to the elicited player is given through
\emph{reveal} calls prior to the elicit call (on top of the public
registers being shared, and the omniscient players having access to
all information). 
\item The constraints are enforced by \emph{ensure} calls that follow an
\emph{elicit}; if the ensure condition fails, execution jumps back
to the last elicit call coming before the current ensure. 
\item For practical playability reasons (i.e. so that games don't stop too
often in practice), we allow each player a small number ($10$) of
ensure fails.
\end{itemize}
\item The \emph{reward} calls are slightly changed compared to the description
of Section \ref{subsec:game-operations}:
\begin{itemize}
\item As discussed above, $black$ and $white$ are a zero-sum pair: any
reward amount $\rho$ given to one player automatically comes with
negative reward amount $-\rho$ to the other one; either player can
be used for a one-player perfect information game (the rewards to
the other are simply discarded if they don't play). 
\item In addition to the values of the rewards being disclosed to the players,
for each \emph{xent} that is part of a reward, the corresponding \emph{atomic
xents $axent$} are shared with the player. If $s$ is $m$-token
long and $t$ $n$-token long, then $axent\left(s,t\right)$ is the
$m$-dimensional vector with the $i$-th coordinate of $axent\left(s,t\right)$
being defined as $-\log\mathbb{P}_{\mathcal{J}}\left\{ X_{n+i}=s_{i}|t_{1},\ldots,t_{n},s_{1},\ldots,s_{i-1}\right\} $. 
\end{itemize}
\item As discussed above, the string space operations are not allowed to
modify $a,b,c$-type registers. 
\begin{itemize}
\item The modification of the string registers is made by \emph{assign}
calls. 
\item The cat operation is denoted by a $+$ and the cut operations by $//$
and $\%$ as explained in Section \ref{subsec:string-space-and-operations}. 
\end{itemize}
\end{itemize}

\subsubsection{\label{subsec:xgl-register-and-instruction-set}XGL Register and
Instruction set}

As discussed above, XGL is an assembly-like language. Lines are executed
sequentially (top-to-bottom), with each line containing a single instruction.
XGL instructions update the registers and interact with the players.
The registers and pre-defined variables set are the following:
\begin{itemize}
\item The $20$ mutable string-space registers of types $s,t,x,y,p$.
\item The $12$ constant string registers of type $a,b,c$. 
\item The $6$ player variables: $black$, $white$, $alice$, $bob$, $carol$,
and $env$. 
\item The $2$ beacon registers: $flag\_1$ and $flag\_2$
\item The following variables are not directly usable, but updated and accessed
by the commands: the current line, the number of lines executed so
far, the number of inner repeats remaining, the number of outer repeats
remaining.
\end{itemize}
Each instruction corresponds to a line of XGL. The syntax is such
that each line can be interpreted as a Python $3+$ line (leveraging
some syntactic sugar). During the game runtime, we have the following
$8$-instruction set
\begin{itemize}
\item $elicit$: updates string space from player's moves. For instance
$elicit(x,20)$ asks for input of max length 20 from player $black$
(if not specified, the player is $black$) and stores the result into
$x$, and $elicit(alice,x1,x2,x3,10)$ asks for $3$ inputs $x1,x2,x3$
from $alice$ of lengths at most 10. 
\item $ensure$: enforce conditions on the player's moves (following an
$elicit$). The $ensure$ calls take a list of boolean expressions
as arguments, in particular those returned by $is\_true$ calls (see
Section \ref{subsec:xent-based-functions} below). For instance, $ensure(is\_true("num\_words<10",s))$
asks the model $\mathcal{J}$ to determine whether $s$ contains fewer
than $10$ words.
\item $reward$: rewards a player (by default $black$) by taking $xent$-based
signed sums (made of sums/differences of $xent,nex,xed,dex$, see
Section \ref{subsec:xent-based-functions} below) as inputs. For instance, 
\begin{itemize}
\item $reward(xent(s|t))$ gives $xent\left(s|t\right)$ to $black$ (the
default player),
\item $reward(xed(s|t))$ gives $xed(s|t)=xent(s)-xent(s|t)$ to $black$,
\item $reward(alice,dex(s|t))$ gives $-xed(s|t)$ to $alice$, 
\item $reward(black,nex(s|t))$ rewards $-xent\left(white,s|t\right)$ to
$black$ and $xent(s|t)$ to $white$ (since $black$ and $white$
are in a zero-sum coupling). 
\end{itemize}
\item $assign$: updates string space (from string space itself, via cat/cut).
For instance $assign\left(s=s1+s2\right)$ concatenates $s1$ and
$s2$ and puts the result into $s$, and $assign\left(s0=c0//s\right)$
cuts $c0$ at the first occurrence of $s$ and puts what comes before
into $s0$, while $assign(t0=c0\%s)$ puts what comes after into $t0$. 
\item $reveal$: marks string space registers to be shared with player (ahead
of an $elicit$). For instance $reveal(alice,s2+t2)$ shares the string
$s2+t2$ with Alice. 
\item $beacon$: plants a flag. The only two possible calls are $beacon(flag\_1)$
and $beacon(flag\_2)$. 
\item $replay$: jumps execution to a previous flag. For instance, $replay(flag\_1,10)$
jumps to $flag\_1$ at most $10$ times. 
\end{itemize}

\subsubsection{\label{subsec:xent-based-functions}Xent-based functions }

The instructions $reward$ and $ensure$ defined above are based on
$xent$ calls to the judge model $\mathcal{J}$. The $xent$ function
takes up to three arguments $s,t,o$ where $t$ and $o$ are by default
empty strings.
\begin{itemize}
\item We have $xent\left(s|t,o\right)=-\log\mathbb{P}_{\mathcal{J},o}\left\{ s|t\right\} $
where $o$ is a general set of customizable pre-instructions for $\mathcal{J}$
that that will come before the pre-prompt $t$. Note that $o$ must
be an in-line constant. By default, if $t$ is empty $xent\left(s,o\right)=-\log\mathbb{P}_{\mathcal{J},o}\left\{ s\right\} $;
if $s$ is empty, $xent\left(s|t,o\right)$ is zero. 
\end{itemize}
Similarly
\begin{itemize}
\item \textbf{\emph{$nex=-xent$}} is the negative \emph{$xent$.}
\item $xed(s|t,o)=xent(s,o)-xent(s|t,o)$ is the \emph{xent }difference:
it encodes how much information $t$ contains about $s$ (from the
point of view of the model $\mathcal{J}$).
\item \emph{$dex(s|t,o)=xent(s|t,o)-xent(s,o)=-xed(s|t,o)$} is the negative
$xed$. 
\end{itemize}
The above quantities can naturally be added and subtracted (but not
multiplied) before being submitted as a \emph{reward}. When they are
rewarded to a player, the detailed atomic \emph{axent} data is also
shared with the player (and all the omniscient players). 

The \emph{$ensure$} calls take booleans which can be the result of
explicit \emph{xent} comparisons (e.g. \emph{$ensure(xent(s|t)<xent(s))$})
or implicit ones, via built-in $is\_true$ and \emph{$is\_false$}
calls:
\begin{itemize}
\item \emph{$is\_true(\lyxmathsym{\textquotedblleft}statement\lyxmathsym{\textquotedblright},params)$}
compares the \emph{xent} of the tokens \emph{$\lyxmathsym{\textquotedblleft}true\lyxmathsym{\textquotedblright}$}
vs \emph{$\lyxmathsym{\textquotedblleft}false\lyxmathsym{\textquotedblright}$}
according to $\mathcal{J}$ after the content of \emph{$\lyxmathsym{\textquotedblleft}statement\lyxmathsym{\textquotedblright}$}
followed by the \emph{params }and a standardized pre-prompt. For instance:
\emph{$is\_true(\lyxmathsym{\textquotedblleft}num\_words<10\lyxmathsym{\textquotedblright},s)$}
can be implemented with a pre-prompt \emph{r} \emph{``Is the statement
`}$num\_words<10$\emph{' about the sentence ''} + \emph{s} + \emph{``
true or false? It is ''}, which is then used to compare the xent
of the \emph{$\lyxmathsym{\textquotedblleft}true\lyxmathsym{\textquotedblright}$}
versus that of \emph{$\lyxmathsym{\textquotedblleft}false\lyxmathsym{\textquotedblright}$}).
The result would be true if \emph{$\lyxmathsym{\textquotedblleft}true\lyxmathsym{\textquotedblright}$}
has a lower \emph{xent} than \emph{$\lyxmathsym{\textquotedblleft}false\lyxmathsym{\textquotedblright}$},
given the statement as a prefix.
\item $is\_false$ returns the opposite of $is\_true$. 
\item By default $ensure(statement)$ corresponds to $ensure(is\_true(statement))$
\end{itemize}
The \emph{$is\_true$} and \emph{is\_false} calls thus use the model
$\mathcal{J}$ as arbiter of the truth as far as satisfying the constraints.
Note that in practice, some truth statements can be implemented otherwise
(e.g. using some hard-coded functions) for efficiency without changing
the game's logic; from a theoretical standpoint, the important point
is that the statements can be arbitrated in principle by the judge
model $\mathcal{J}$. 

\subsubsection{\label{subsec:xgl-metadata-and-pregame-code}XGL Metadata and Pre-Game
Code}

Most of the XGL metadata consists of data to link $\mathcal{J}$ to
a specific LLM, and to link the players to various LLM agents, as
discussed Xent Game description above (Section \ref{subsec:game-metadata}).
This data is set in a dictionary (if modifications to default values
are needed) and it includes the prompts to fill the \emph{a}, \emph{b},
\emph{c} registers. The only part of the metadata that is absent is
the session seed, which is set externally. Like the code, the metadata
is shared with all players. 

The only difference with the list of Section \ref{subsec:game-metadata}
is that much of the metadata described there appears in-line in the
game code via the use of quote-marked string constants. This choice
is to made to improve readability for the players.

One important design element is that the metadata can be sequentially
defined by several dictionaries, where each dictionary can override
the variables defined by the dictionaries that came before in the
order presented below:
\begin{enumerate}
\item for the whole space;
\item for a game subspace (e.g. when running a benchmark);
\item for a specific game;
\item for a specific running instance (e.g. for the random seed).
\end{enumerate}
Note that the in-line metadata cannot be overridden with this scheme. 

\subsubsection{\label{subsec:xgl-as-a-target-language}XGL as a Target Language
for LLMs}

In addition to readability and ease of execution, XGL is designed
from the beginning to facilitate the \emph{automatic generation} of
games, which is instrumental in Section \ref{subsec:scope-growth}
below to \emph{discover new games} relevant to measuring general model
abilities. Put simply, XGL is designed to be a good target language
for LLM generation. The structure of the language is designed to help
with the spontaneous discovery of new games by abiding by the following
principles:
\begin{itemize}
\item An XGL program is valid if and only if all its lines are individually
valid XGL lines. This allows one to e.g. `cross-breed' games.
\item The structure of the $reward$ and $ensure$ are similar and they
both take xents as inputs. 
\item All the variables are named in advance and they are string typed only. 
\end{itemize}

\subsection{\label{subsec:xent-graphical-language}Graphical Representation:
Xent Games as Diagrams}

In order to study simple Xent Games, especially the ones written in
XGL, it is often convenient to express their operational logic via
graphical diagrammatic representations: this is the way that the examples
of Section \ref{sec:xent-game-examples} are represented. 

In this subsection, we present the Xent Game Diagram (XGD) representation
scheme that makes simple games particularly easy to grasp: 
\begin{itemize}
\item This scheme is useful (for humans) to understand the content of a
game, and to create new games as well. 
\item It makes the games with \emph{$black$} and \emph{$white$} particularly
easy to draw. 
\item It is also a good way to grasp the nature of the Xent Game space,
in particular the interactions between the string space operations,
the xent operations, and the player's moves are much easier to understand. 
\item While the assembly-like nature of XGL is usually easy to follow, the
graphical representation eliminate some redundancies involved with
the notation (e.g. swapping the name of two registers leaves a game
unchanged, it is not clear which registers are in fact used).
\end{itemize}
The philosophy of XGD is quite simple and in some sense closer to
the theoretical move description (see Section \ref{subsec:moves-and-constraints}).
The XGD representation is based on symbols linked via edges of various
stroke styles, surrounded by constraint boxes. See Figure \ref{fig:xgd-primitives}
for a reference of the symbols, and Figures \ref{fig:xgd-human-friendly-1p-games},
\ref{fig:xgd-1p-xent-games}, \ref{fig:xgd-interception-games}, \ref{fig:xgd-naive-chess},
and \ref{fig:xgd-secret-sharing} for some simple examples of Xent
Games. Note that it is easier to understand XGD by looking at these
examples first. 

\begin{figure}
\includegraphics[scale=0.8]{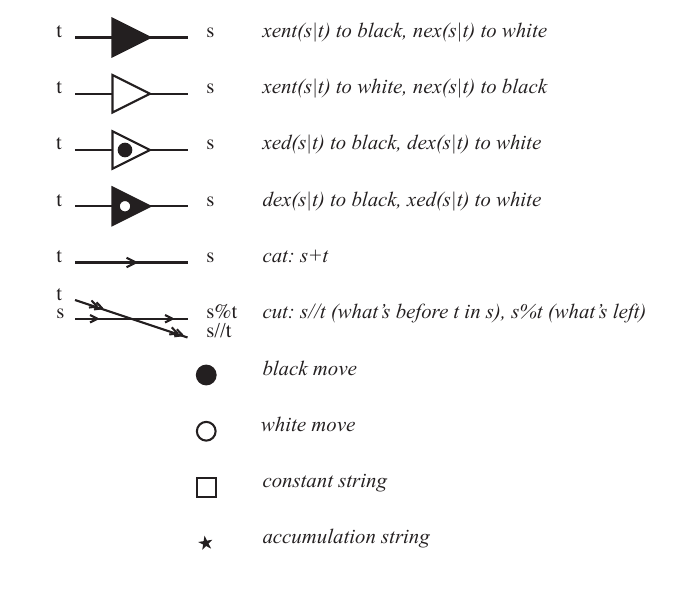}\caption{\label{fig:xgd-primitives}XGL primitives represented in XGD}
\end{figure}

\subsubsection{Colors and String Space Variables}
\begin{itemize}
\item Associated with each player is a \emph{color} and a \emph{negative
color}:\emph{ }the players \emph{$black$, $white$, $alice$, $bob$,
$carol$, }$david$, and\emph{ env} have the (positive) colors \emph{$black$,
$white$, $red$, $green$, $blue$,} $and$ \emph{purple} respectively,
and they have the negative colors \emph{$white$, $black$, $cyan$,
$magenta$,} \emph{$yellow$, }and \emph{$olive$} respectively. 
\item Each string space variable state is represented by a node: 
\begin{itemize}
\item A colored circle $\circ$ for \emph{$elicit$} inputs, with a number
inside the circle representing max length
\item A square ${\small \square}$ for instance-set constant strings.
\item A star $\star$ for intermediate variable assignments. 
\end{itemize}
\end{itemize}

\subsubsection{Xent Symbols}

The \emph{xent symbols} are the symbols associated with \emph{xent}-based
functions \emph{$xent$, }$nex$, $xed$, and $dex$. The xent symbols
are each represented by a colored triangle-like symbol, with each
color (positive or negative) corresponding to the players.
\begin{itemize}
\item For the xent symbols that are not in an \emph{$ensure$} box (see
below):
\begin{itemize}
\item a black triangle represents \emph{$xent$} awarded to \emph{$black$}
and $nex$ to \emph{$white$} (since \emph{$black$/$white$} are
in zero-sum pairing); a white triangle (with a black boundary, for
readability) represents \emph{$nex$} to \emph{$black$} and \emph{$xent$}
to $white$.
\item a white triangle with a black dot represents \emph{$xed$} awarded
to \emph{$black$} and \emph{$dex$} to $white$; a black triangle
with a white dot represents \emph{$xed$} awarded to $white$ and
$dex$ to $black$.
\item the orientation of the triangle is from the prefix string to the string
whose xent function is computed, i.e. a black triangle points from
\emph{$t$} to $s$, \emph{$xent(s|t)$} is awarded to \emph{$black$}
and \emph{$nex(s|t)$} to $white$. 
\end{itemize}
\item Each xent symbol in an $ensure$ box is colored similarly, and gives
a constraint on the move of the \emph{$elicit$} contained in the
same box: the sum of the xent symbols in the box must be positive
for the $ensure$ to pass. 
\item When needed, each xent symbol comes with a number written on its bottom
right, specifying the order in which it plays a role. 
\item Each xent symbol is associated with a player via their color and negative
color (see Figure \ref{fig:xgd-primitives}). 
\end{itemize}

\subsubsection{Elicit and Reveals}

As discussed above $elicit$ instructions are represented by circles.
\begin{itemize}
\item The colors refer to the elicited player's (positive) color, e.g. $\bullet$
for $black$ and $\circ$ for $white$. 
\item Inside of the elicit circles, numbers can be written, denoting the
maximal length of the elicit.
\item The revealed information, if we elicit a non-omniscient player, is
represented by star of the relevant color appearing to the left (i.e.
having been seen) of the $elicit$ box.
\item The numbers on the bottom right, that specify the order of operations,
are particularly important for elicits.
\item In case of an imperfect-information game, undirected edges link the
elicit box to the various strings shared with the player being elicited.
\end{itemize}

\subsubsection{Ensure Boxes}

The $ensure$ constraints are represented by referring to elicit symbols.
An $elicit$ can have several corresponding \emph{ensure} statements;
in the case an $ensure$ contains several $elicit$ statements, the
$elicit$ statement which execution jumps back to if the $ensure$
fails is surrounded by a circle. 
\begin{itemize}
\item Either include xent signed rewards with the colors associated with
the player being elicited (as above). 
\item Truth conditions to be judged by the $\mathcal{J}$ model. These are
typically written in natural language, and may refer to local variables
denoted by $\star$ symbols. 
\end{itemize}
The $ensure$ boxes are implicitly linked by strokes to at least one
$elicit$ box, and if they fall back on the last elicit if the $ensure$
fails. 

\subsubsection{Cat/Cut}

The string operations (in particular cat/cut) are represented as follows:
\begin{itemize}
\item The concatenation of two strings (with a default spacing character
between them) is represented by an edge linking the two corresponding
nodes oriented with an arrow symbol, determining the order of the
concatenation ($s\to t$ represents $s+t$, i.e. $s$ followed by
$t$).
\item Each cat/cut operation corresponds to a certain \emph{stroke type}
(e.g. regular, dashed, wiggly, doubled, etc.)
\item The concatenation of $n$ strings $s_{1},\ldots,s_{n}$ is represented
by $n-1$ oriented edges $s_{j}\to s_{j+1}$ for $0<j<n$ that use
the same stroke type .
\item The \emph{cut} of a string $s$ by another $t$ is represented by
a single-arrow oriented edge being `intercepted' by a double-arrow
oriented edge (with same stroke type). The emerging double-arrow edge
carries $s//t$, while the emerging single-arrow edge carries $s\%t$.
\item A string can belong to several concatenations with several edges being
incident if they are represented using different stroke styles.
\item An unoriented edge with a certain stroke type `carries' the result
of the string operation `performed' by the corresponding stroke type
to another operation (e.g. a xent symbol). 
\item A string can be stored in a star if there is an unoriented edge `carrying'
the string towards the star; there should be exactly one unoriented
edge incident to a star.
\item Dotted (unoriented) edges carry values from a node to another, with
the orientation being from left to right. 
\end{itemize}

\subsubsection{Order of Operations}
\begin{itemize}
\item The order of operations ($assign$, $elicit$, $reward$) is dictated
(in case of ambiguity) by numbers at the bottom right of the boxes.
\item In case of repeats, a REPEAT($k$) marker is added to specify which
way to go in when repeating the first $k$ time, and an AFTER marker
to specify which way to go at the end of the repeat. 
\end{itemize}

\section{\label{sec:xent-game-examples}Xent Game Examples}

The goal of this section is to illustrate the richness of the space
of Xent Games introduced in Section \ref{sec:xent-games} above, through
a number of examples of interest. Note that the goal of this section
is mainly to illustrate the sophisticated nature of such games in
terms of the skills they require; while each example captures a diversity
of interesting features that are intuitively understandable, it is
not claimed that the set of games presented is in any sense exhaustive
or balanced as far as measuring the abilities of LLMs. In that regard,
it is also fundamental to understand that none of \emph{these games
are expected to be played optimally by any present or future model}
(or even less to be solvable analytically); inasmuch they are concerned
with model benchmarking, the goal of such games is merely to \emph{compare
the behaviors of various models} with respect to one another. 

This section is organized as follows:
\begin{itemize}
\item In Section \ref{subsec:one-player-combinatorial-games}, a number
of one-player perfect information games are presented: these are simply
discrete optimization problems associated with the Xent Measure of
the judge model $\mathcal{J}$.
\item In Section \ref{subsec:two-player-zero-sum-combinatorial-games},
a number of two-player perfect information zero-sum games are presented:
informally, these are minimax games, i.e. since each player's gain
is the other player's loss, theoretically the optimal play is to minimize
the other player's return over all their possible actions following
one's own action. 
\item In Section \ref{subsec:imperfect-information-games}, we present a
few multi-player general-sum imperfect information games, aimed at
highlighting a number of additional challenges associated with such
games. It should be noted that this class is extremely vast and that
the examples should not be expected to cover even a substantial fraction
of the additional challenges raised by such games. 
\end{itemize}

\subsection{\label{subsec:one-player-combinatorial-games}One-Player Combinatorial
Games}

One-player perfect information games are simply combinatorial optimization
problems involving the xent function. 

\begin{figure}
\begin{lstlisting}
assign(s=story())
elicit(t, 10)
ensure("no common words between"+s+"&"+t)
# asks for a nontrivial "summary" of s
reward(xed(s|t))
\end{lstlisting}

\begin{lstlisting}
assign(s1=story(),s2=story(),s3=story())
elicit(t, 10)
ensure("no common words between"+s1+s2+s3+"&"+t))
# asks for a nontrivial "joint summary" of s1, s2, s3
reward(xed(s1|t)+xed(s2|t)+xed(s3|t))
\end{lstlisting}

\begin{lstlisting}
assign(s1=story(), s2=story())
elicit(t, 10)
ensure("no common words between"+t+"&"+s1+s2)
reward(xed(t|s2+s1)+xed(s2|s1+t)+xed(s1|t+s2)) 
\end{lstlisting}

\caption{\label{fig:human-friendly-one-player-games}Three human-friendly one-player
Xent Games (played by $black$, in perfect information mode as per
the default)}

\end{figure}

\begin{figure}

\includegraphics[scale=3]{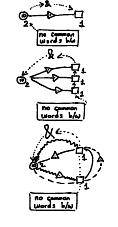}\caption{\label{fig:xgd-human-friendly-1p-games}XGD illustrations of the games
of Figure \ref{fig:human-friendly-one-player-games}}
\end{figure}

In Figure \ref{fig:human-friendly-one-player-games}, we present three
human-friendly optimization games which involve the $xed$ function
(with $xed\left(s|t\right)=xent\left(s\right)-xent\left(s|t\right)$).
Note that the first game leads to results that would plausibly be
a desirable output from an LLM while being quite hard to solve in
practice (taking a small temperature leads to a greedy approximation
of the solution). 

\begin{figure}
\begin{lstlisting}
assign(s=story())
elicit(t, 10)
assign(s1=s//t,s2=s%t)
# rewards the simplest cut possible 
# that makes what comes before
# given what comes after
reward(xed(s1|s2)-xent(t))
\end{lstlisting}

\begin{lstlisting}
assign(s=story())
elicit(t, 10)
assign(s1=s//t,s2=s%t)
# rewards the cut that makes s1 
# much more likely to follow s2 than vice versa
reward(xed(s1|s2)-xed(s2|s1)) 
\end{lstlisting}

\begin{lstlisting}
assign(s1=story(), s2=story())
elicit(t1, 10)
ensure("no common words between"+s1+s2+"&"+t1)
elicit(t2, 10)
ensure("no common words between"+s1+s2+"&"+t2)
# rewards maximally unrelated prefixes 
# that are good prefixed for both stories
reward(xed(s1|t1)+xed(s2|t2)+xed(s1|t2)+xed(s2|t1)-xed(t1|t2)-xed(t2|t1))
\end{lstlisting}

\begin{lstlisting}
assign(s=story())
elicit(t, 10)
# find a good prefix t for s 
# that is unlikely given s
reward(xed(s|t)-xed(t|s)-xent(t))
\end{lstlisting}

\begin{lstlisting}
assign(s1=story(), s2=story())
elicit(t1, 10)
ensure("no common words between"+t1+"&"+s1+s2)
elicit(t2, 10)
ensure("no common words between"+t2+"&"+s1+s2)
# make the story s1->t1->t2->s2->s1 maximally
# plausible (with the no common words constraint)
reward(xed(s1|t2+s2+t1)+xed(t2|s2+t1+s1))
reward(xed(s2|t1+s1+t2)+xed(t1|s1+t2+s2))
\end{lstlisting}

\begin{lstlisting}
assign(s1=story(), s2=story())
elicit(t1, 10)
elicit(t2, 10)
# find well-explained continuations t1, t2 that are not explaining
# such that t1 is an explanation for t2 but not vice versa
reward(xed(t1|s1)-xed(s1|t1)+xed(t2|s2)-xed(s2|t2))
reward(xed(t2|s1)-xed(s1|t2)+xed(t1|s2)-xed(s2|t1))
reward(xed(t2|t1)-xed(t1|t2))
\end{lstlisting}

\caption{\label{fig:1p-xent-games}Some simple examples of one-player combinatorial
Xent Games }
\end{figure}

In Figure \ref{fig:1p-xent-games}, we present a small collection
of games each of which comes with a slightly different interpretation,
despite fairly short codes and limited number of involved variables. 

\begin{figure}

\includegraphics[scale=2]{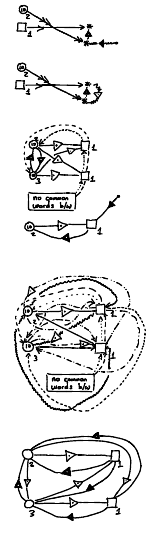}

\caption{\label{fig:xgd-1p-xent-games}XGD representations of the games of
Figure \ref{fig:1p-xent-games}}

\end{figure}

\subsection{\label{subsec:two-player-zero-sum-combinatorial-games}Two-Player
Zero-Sum Combinatorial Games}

\subsubsection{Interception/Selection Games}

\begin{figure}

\begin{lstlisting}
assign(s=story())
elicit(white, t, 20)
elicit(black, t1, 10)
ensure("No common words between" + t + "&" + t1)
# white tries to intercept a priori 
# the words that black could use
reward(black, xed(s|t1)+xed(t1|s))
\end{lstlisting}

\begin{lstlisting}
assign(s1=story(), s2=story())
elicit(white, t1, 10)
elicit(white, t2, 10)
elicit(black, t, 10)
# black (who receives the negative of white's score) 
# must try to "derail" the story 
# that white is trying to build, 
# while balancing local consistency
reward(white, xent(s1+t1+t+t2+s2)-xent(s1+t1+t)-xent(t+t2+s2)+xent(t))
\end{lstlisting}

\caption{\label{fig:interception-games}Interception Games}

\end{figure}

\begin{figure}
\includegraphics[scale=4]{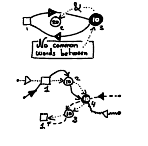}

\caption{\label{fig:xgd-interception-games}XGD representations of the games
of Figure \ref{fig:interception-games}}

\end{figure}

Some interesting classes of zero-sum perfect information games are
interception games, where one player must anticipate the other player's
moves and hinder their possibilities. 

\subsubsection{Chess}

Another interesting class of games are (variants of) sophisticated
games that are (typically) expected to be known to some degree by
the judge model $\mathcal{J}$. For instance, the game of chess is
in fact very easy to formulate as a Xent Game, given a sophisticated
enough judge model able to recognize valid chess moves. 

\begin{figure}
\begin{lstlisting}
assign(s="")
beacon(flag_1)
elicit(white, t, 4)
ensure("Is valid chess game, black to play or game over:" + s + t)
assign(s=s+t)
elicit(black, t, 4)
ensure("Is valid chess game, white to play or game over:" + s + t)
assign(s=s+t)
replay(flag_1, 20)
reward(white, xent("In the game" + s + "the most likely winner is black"))
reward(black, xent("In the game" + s + "the most likely winner is white"))
\end{lstlisting}

\begin{lstlisting}
assign(s="")
beacon(flag_1)
elicit(white, t, 6)
ensure("Is valid chess game, black to move or game over (specify winner):" + s+t)
assign(s=s+t)
elicit(black, t, 6)
ensure("Is valid chess game, white to move or game over (specify winner):" + s+t)
assign(s=s+t)
replay(flag_1, 5000)
reward(white, xed(s // "#", "white="))
reward(black, xed(s // "#", "black="))
\end{lstlisting}

\caption{\label{fig:naive-chess}Naive Chess and Chess}
\end{figure}

As described in Figure \ref{fig:naive-chess}, a naive way to play
chess using Xent Games is to play 20 moves and and use the judge model's
cross-entropy difference between the estimates for `white is winning'
vs `black is winning' to determine the winner. The judge must enforce
the validity of moves at every step. Alternatively, the full game
of chess can be implemented, but since the theoretical upper bound
on chess game lengths is around 5000, the game must be allowed to
run for all these steps. 

\begin{figure}
\includegraphics[scale=2.5]{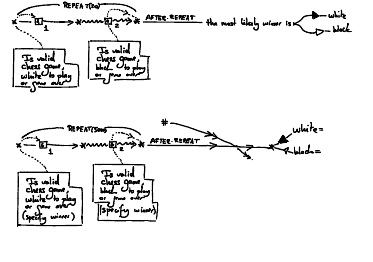}

\caption{\label{fig:xgd-naive-chess}XGD representations of the Naive Chess
and Chess games of Figure \ref{fig:naive-chess}}

\end{figure}

\subsubsection{Naive Sprig: Concise Mathematical Proof Debate}

Another interesting class of games revolves around (concise) mathematical
proofs. The naive idea is to build a proof that can be debated by
two parties (inspired by e.g. the Sprig protocol \cite{cghcl-i,cghcl-ii}).
Note that for this game to correspond in reality to true mathematical
debates, a very strong judge model would be needed. 

\begin{figure}

\begin{lstlisting}
assign(s=story("Get a provable elementary mathematical statement"))
assign(s3=story("Get three typical one-line math arguments"))
assign(s0="Here is a proof debate about the statement:" + s)
assign(t0="")
beacon(flag_1)
elicit(white, t, 100)
ensure(xent(t|s0)<xent(s3))
assign(s0=s0+"white proposes the following prove the last statement:" + t)
elicit(black, t0, 100)
elicit(black, s, 100)
ensure(xent(s|s0)<xent(s3))
ensure(s + " is a possible point of contention about the debate " + s0)
assign(s0=s0+"black raises a point of contention about " + s0)
replay(flag_1, 100)
elicit(black, t0, 100)
ensure(t0 + "is white's last answer about a point of contention in the debate")
reward(white, xent(t0 + "is a correct math proof w/o missing details: false"))
reward(black, xent(t0 + "is a correct math proof w/o missing details: true"))
\end{lstlisting}

\caption{Mathematical debate}

\end{figure}

\subsection{\label{subsec:imperfect-information-games}Imperfect Information
Games}

Beyond the perfect information setup, the imperfect information games
bring considerably more sophistication for the same space of allowed
moves. We present a few games that are illustrative of the challenges
associated with imperfect information, involving in particular information
sharing, coordination, guessing, and bluff. 

\subsubsection{Secret Sharing}

A simple example of an imperfect-information game is secret sharing
(Figure \ref{fig:secret-sharing-game}), where Alice must split a
secret into Carol's and David's share so that neither can find the
secret, but so that Bob, who has both shares, can find it. Notice
that these kind of problems have simple cryptographic solutions, but
that playing such a game with LLMs can still give interesting results. 

\begin{figure}
\begin{lstlisting}
assign(s=story())
reveal(s, alice)
elicit(alice, s2)
elicit(alice, s3)
reveal(bob, s2)
reveal(carol, s3)
reveal(david, s2)
reveal(david, s3)
elicit(david, t0)
elicit(bob, t2)
elicit(carol, t3)
reward(bob, xed(s|t2))
reward(carol, xed(s|t3))
reward(david, xed(s|t0))
reward(alice, xed(s|t0)+xed(s|t0)-xed(s|t2)-xed(s|t3))
\end{lstlisting}

\caption{\label{fig:secret-sharing-game}Secret Sharing Game}

\end{figure}

\begin{figure}
\includegraphics[scale=4]{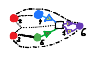}

\caption{\label{fig:xgd-secret-sharing}XGD representation of the Secret Sharing
Game of Figure \ref{fig:secret-sharing-game}}

\end{figure}

\subsubsection{Coordination}

Another example of an imperfect information game is coordination (Figure
\ref{fig:coordination-game}), where Bob and Carol have access to
the same information and must produce the closest possible guesses. 

\begin{figure}

\begin{lstlisting}
assign(s=story())
reveal(alice, s)
reveal(bob, s)
reveal(carol, s)
elicit(alice, t1, 10)
elicit(bob, t2, 10)
elicit(carol, t3, 10)
reward(alice, xed(s|t1)+xed(s|t1)-xed(t1|x)+xed(t1|t2)+xed(t1|t3)-xent(t1))
reward(bob, xed(s|t2)+xed(s|t2)-xed(t2|x)+xed(t2|t1)+xed(t2|t3)-xent(t2))
reward(carol, xed(s|t3)+xed(s|t3)-xed(t3|x)+xed(t3|t1)+xed(t3|t2)-xent(t3))
\end{lstlisting}

\caption{\label{fig:coordination-game}Coordination Game}

\end{figure}

\subsubsection{Guessing Games}

Finally, a very rich family of games are guessing games, where one
or several players try to learn information by repeated interactions
(see Figure \ref{fig:guessing-games}). 

\begin{figure}
\begin{lstlisting}
assign(s=story(8))
beacon(flag_1)
elicit(alice, t, 8)
reward(alice, xed(t|s))
replay(flag_1, 100)
\end{lstlisting}

\begin{lstlisting}
assign(s=story(), s0=story(4))
beacon(flag_1)
elicit(alice, t, 4)
reward(alice, xed(s|t))
reward(alice, -xed(s|t))
replay(flag_1, 100)
reward(alice, xed(t|s0)+xed(s0|t))
\end{lstlisting}

\begin{lstlisting}
assign(s=story())
reveal(alice, s)
elicit(alice, t, 10)
reveal(bob, t)
elicit(bob, s0, 10)
reward(bob, xed(s0|t0))
\end{lstlisting}

\begin{lstlisting}
assign(s=story(), s0=story())
reveal(alice, s)
reveal(alice, s1)
elicit(alice, t, 10)
reveal(carol, t)
elicit(carol, t0, 10)
# carol has the power to strip the message from its content
assign(t=t//t0) 
reveal(bob, t)
reward(alice, xed(s|t))
elicit(bob, t1, 10)
# alice and bob get double points for passing information about s and for s0
reward(alice, xed(s|t1)+xed(t1|s)+xed(s|t1)+xed(t1|s)+xed(s0|t1)+xed(t1|s0))
reward(bob, xed(s|t1)+xed(t1|s)+xed(s|t1)+xed(t1|s)+xed(s0|t1)+xed(t1|s0))
# carol is incentivized to let information about s pass
# she is doubly punished for information about s1 
reward(carol, xed(s|t1)+xed(t1|s)-xed(s1|t1)-xed(t1|s1)-xed(s1|t1)-xed(t1|s1))
\end{lstlisting}

\caption{\label{fig:guessing-games}Guessing Games}
\end{figure}

\section{\label{sec:xent-game-space-model-play-properties}Xent Game Space:
Model Play Properties}

Being made of text interactions mediated by LLM judges, Xent Games
introduced in Section \ref{sec:xent-games} are naturally designed
to be played by LLM-based agents. In this section, we focus on the
relationship between Xent Game play and LLMs. 

\subsection{\label{subsec:xent-game-eval-mode}Xent Game Eval Mode}

By definition, a Xent Game $G$, beyond its game operation code, involves
some metadata $D$ including `linking' metadata: 
\begin{itemize}
\item the specification of the judge model $\mathcal{J}$;
\item the NPC models linked to the set $\mathcal{A}\left(G\right)$ of\emph{
active }players (i.e. the ones which are involved in at least one
$elicit$ statement). 
\end{itemize}
When using a Xent Game to evaluate a model $\mathcal{M}$, the idea
is to `freeze' a game $\left(G,D\right)$, pick an active player $P\in\mathcal{A}\left(G\right)$
in the game $G$ and to consider the family of games $\left(G,D|_{P=\mathcal{M}}\right)$
where $D|_{P=\mathcal{M}}$ denotes the linking of $P$ with the model
$\mathcal{M}$, overriding the linking of $\mathcal{P}$ in $D$,
i.e. forcing the player $P$ to become the main character played by
$\mathcal{M}$. Informally, this corresponds to the following simple
idea: consider a game $G$, pick a player $P$, and measure how the
model plays the role of the `character' $P$.

Running a Xent Game in Eval Mode furthermore involves specifying a
random `map-seed' $\sigma$ used to set the instance random data (unless
we deal with a game with deterministic instances, e.g. Chess or Go),
and in the generation of the outputs of the players except the model
being evaluated; for the model being evaluated, we typically use a
different `play-seed' $\varsigma$. 

Hence, when discussing the (unnormalized) score of a given model $\mathcal{M}$
on a game $G$, we consider the averaging over random values $\sigma$
of the reward in the game $\left(G,D|_{P=\mathcal{M},\mathcal{S}}\right)$
where $P$ is overridden to be played by $\mathcal{M}$ and the seed
is set to random values $\sigma_{i}$ for $i=1,\ldots,N$:
\begin{equation}
S_{G}\left(\mathcal{M}\right):=\frac{1}{N}\sum_{i=1}^{N}S_{G}\left(G,\sigma_{i},\varsigma_{i}\right).\label{eq:average-score}
\end{equation}

\begin{rem}
This definition is closely related to that of the $arms$ definition
given in Section \ref{subsec:playability-few-shot-and-fine-tuning-definitions}
below. It is important in practice, for us to not let $N\to\infty$
(which would yield an expectation), as this leads to undesirable features
(see Remark \ref{rem:why-not-expectation}). 
\end{rem}

\begin{rem}
In order to allow for unified normalization across various games,
game maps, and game settings, the game scores should be normalized
using a base model (see Section \ref{subsec:score-normalization}
below). 
\end{rem}

\subsubsection{\label{subsec:game-representation-and-variance}Game Representation
and Variance}

When averaging over map seeds in practice, it is important to take
into account the variance of rewards with respect to the map seed;
this sometimes singles out certain formulations of games that otherwise
have the same optimal solutions and strategies. 

For instance, the basic reverse prompting game (Example \ref{exa:reverse-prompt}
above) admits the following equivalent representations: when the reward
$-xent\left(s|t\right)$ is given to $white$ and when $xed\left(s|t\right)=-xent\left(s|t\right)+xent\left(s\right)$
is given to $white$: the scores as a function of the move $t$ just
differ by a constant. The latter formulation is typically to be preferred,
as its variance with respect to the randomness of $s$ (induced by
the map-seed $\sigma$) is lower; the reliance of the game designs
presented in Section \ref{sec:xent-game-examples} upon the $xed$
function is largely motivated by this consideration.

\subsubsection{\label{subsec:game-well-posedness}Well-Posedness}

A notion that is related to the one discussed in Section \ref{subsec:game-representation-and-variance}
above is that of \emph{well-posedness}. This can be viewed as an analog
to the notion of boundedness in optimization: trying to maximize a
function that is not bounded from above cannot yield meaningful results,
even if, in practice, values may be bounded by machine precision (whatever
result we will get will reflect something about the machine running
the computations rather than about the problem). In the setting of
Xent Games, rewards on any game instance are naturally bounded by
the bounded sizes of the strings, but that does not imply solutions
are anyhow meaningful, as they may be more reflective of irrelevant
model artifacts rather than anything else. 
\begin{rem}
\label{rem:ill-posed-xent-game}A good example of this is the simple
game that consists in maximizing $xent\left(s\right)$ over all strings
$s$ of a certain length: this leads the model to `tap' into words
with extremely low probability whose estimates have usually no reason
to have been precisely calibrated, and the results will typically
look like random text.
\end{rem}

Precisely defining the notion of well-posedness for Xent Games is
in fact not trivial: as much as it is easy to come up with examples
of ill-posed games, such as the one of Remark \ref{rem:ill-posed-xent-game},
distinguishing between well-posed and ill-posed games can be difficult.
A simple intuitive criterion is that gameplay should not be affected
by very small probabilities, i.e. that if we put a little bit of noise
in the probability vector, the game-play should not be substantially
altered (either by the modification of allowed moves or by the scoring
of the outcome). While this could be formulated into a precise definition,
the following is essentially equivalent and much simpler in practice:
\begin{defn}
\label{def:well-posedness}For $\epsilon>0$, we say a Xent Game $G$
(in eval mode) is $\epsilon$\emph{-well-posed} if the main player's
optimal strategy is not affected by $\epsilon$-clipping of probabilities,
i.e. replacing $x\mapsto-\log x$ by $x\mapsto-\max\left(\log x,\log\epsilon\right)$
in the cross-entropy computation.
\end{defn}

\begin{rem}
In practice, a reasonable (though fairly arbitrary) value we will
take for $\epsilon$ is $1/|\mathcal{V}|$, i.e. we will disregard
games where optimal play relies on tokens that are less likely than
what a random uniform guess would give. 
\end{rem}

\begin{rem}
By this definition, it is not hard to see that the game of Remark
\ref{rem:ill-posed-xent-game} is ill-posed, as is the paradoxical
game of Example \ref{exa:paradox} (when in eval mode with $white$
or $black$). 
\end{rem}

\subsubsection{\label{subsec:game-repetition}Game Repetition}

An important feature of learning to play games is repetition, i.e.
learning from mistakes, and more generally understanding the interplay
between a game play history and the rules of the games to optimize
one's strategy. A useful setting in practice is to allow the to replay
of games by the main character played by the model $\mathcal{M}$,
while keeping the \emph{same map-seed} and the \emph{same play-seed}
over the various replays, giving the past games as context to $\mathcal{M}$,
while keeping no past game context to the other NPC agents. This setup,
which advantages $\mathcal{M}$ over the NPC players, can be expected
to give an increased score as the attempts are repeated for a reasonable
class of games: as discussed in Section \ref{subsec:few-shot-definition-playability}
below, we call these the \emph{few-shot playable games}.

\subsubsection{\label{subsec:game-fine-tuning}Game Fine-Tuning}

Another important way in which we can see a model getting better at
games is by processing examples and adjusting its parameters on them,
i.e. to perform some fine-tuning based on reinforcement learning principles.
We will abstract the many subtle details involved with such a process
by merely assuming that we are given, for some model $\mathcal{M}$,
a process $\mathcal{F}$ that, given a game $G$ (with an assignment
of $\mathcal{M}$ to some player $P$ in $G$), outputs a model $\mathcal{F}_{G}^{\mathcal{T}}\mathcal{M}$
corresponding to `$\mathcal{M}$ fine-tuned to play better at $G$'
for some time $\mathcal{T}$. 

While we intentionally avoid setting specific details about the process,
it is useful to take as an example the one developed in Section \ref{subsec:transfer-value-numerical-example}
below. Note that we do not require that the model has actually gotten
better at $G$: as will be discussed in Section \ref{subsec:playability-few-shot-and-fine-tuning-definitions}
below, this corresponds to the family of \emph{fine-tuning playable
games}. 

It should be noted that a fine-tuning mechanism should not be assumed
to exist for all models $\mathcal{M}$: for instance, some closed
models are only available in inference mode and not be fine-tunable.
In this paper, we simply required the existence of a sufficiently
strong base model $\mathcal{M}_{\mathcal{B}}$ for which a fine-tuning
mechanism is available that makes enough games playable (see Section
\ref{sec:evolution-in-game-space} below). 

\subsection{\label{subsec:playability-few-shot-and-fine-tuning-definitions}Playability:
Few-Shot and Fine-Tuning Definitions}

Central to our investigation of well-posed Xent Games is the notion
of \emph{playability}, which informally corresponds to a quality of
certain games that models can get improved performance on, either
by repeated play context (few-shot definition) or by reinforcement
learning (fine-tuning definition). This should be thought of in opposition
to non-playable games, which contain, for instance, as a non-trivial
example, the breaking of a secure code (even with a small secret length),
as one either succeeds or fails and there is no intermediate reward
and no sense that we get closer to succeeding.

\subsubsection{\label{subsec:few-shot-definition-playability}Few-Shot Definition}

A remarkable ability of modern LLMs is their so-called few-shot learning
abilities \cite{openai-2020,kojima-gu-reid-matsuo-iwasawa}: from
a limited number of feedback examples directly put in their context
window, models are often able to substantially improve their responses.
The notion of few-shot playability by a model $\mathcal{M}$ refers
to those games where the few-shot abilities are reflected in game-play: 
\begin{defn}
\label{def:few-shot-playable-game}For a game $G$ and $N$ random
samples of map seeds and play seeds, we define the average running
max score $\mathrm{arms}_{N,k}\left(G,\mathcal{M}\right)$ as the
(random) sequence defined by
\[
\mathrm{arms}_{N,k}\left(G,\mathcal{M}\right)=\frac{1}{N}\sum_{\ell=1}^{N}\max_{j=1,\ldots,k}\left(S_{j}^{\left(\ell\right)}\right),
\]
where $S_{j}^{\left(\ell\right)}$ is the score at the $j$-th iteration
over the sample $\ell$. 

A well-posed game $G$ is $\left(N,\epsilon,m\right)$-playable in
few-shot mode by a model $\mathcal{M}$ if, with probability $1$-$\epsilon$,
the score $\mathrm{arms}_{N,k}\left(G,\mathcal{M}\right)$ is strictly
increasing for at least $m$ steps. We say that a playable game is
$\left(N,\epsilon,\theta,m\right)$ playable if there exists $m$
thresholds $k_{1}<\cdots<k_{m}$ such that $\mathrm{arms}_{N,k}\left(G,\mathcal{M}\right)$
exceeds $b+j\theta$ for the first time at $k_{j}$ for $j=1,\ldots,m$
where $b$ is the expected score at the first iteration. 
\end{defn}

In practice, we will assume that some values of $N,\epsilon,m$ are
set. In the numerical experiments below (Section \ref{subsec:numerical-example-benchmarking}),
we take $N=20$ and $m=30$.
\begin{rem}
\label{rem:why-not-expectation}This definition means that, in practice,
averaging over a sufficiently large number of samples, we do see the
running max score increase substantially for a while (i.e. at every
step, the model has done better in one sample). We choose this definition
rather than that of an expectation as the latter is sensitive to very
unlikely outcomes (e.g. breaking a secure cipher is not a playable
game, yet the probability of cracking it is technically nonzero at
every step). 
\end{rem}

As we will see in Section \ref{sec:xent-game-measures}, few-shot
playability is an important a priori property for benchmarking: the
benchmarking score of a game will be defined relative to the times
when a base model $\mathcal{M}_{\mathcal{B}}$ crosses $\theta$-levels
for a game's repeated play. 

\subsubsection{\label{subsec:numerical-example-benchmarking}Numerical Example:
Benchmarking with Single-Text}

As an example, the single-text game presented in Figure \ref{fig:human-friendly-one-player-games}
is a good example of a playable game that can be used to compare various
model performances. Our results confirm that current state-of-the-art
LLMs are indeed capable of playing a Xent Game written in XGL and
playable in the sense of Definition \ref{def:few-shot-playable-game}:
they tend to improve over iterations. 

The plots in Figure \ref{fig:performance-4-flagship-models} show
the performances of GPT-4.1, Claude Opus 4, Gemini 2.5 Pro, and Grok
3 when playing the single-text game posed in the Xent Game Language
in an iterated scenario using GPT2 as the judge model. Each test case
is the single text game with a different generated string. While these
specific test cases show clear improvement across the set of test
models, individual game map results are more mixed, showing flat or
even occasionally negative performance as iterations increased. Our
results are obtained from the code available in the \href{https://github.com/xentlabs/xega/}{xega repository}
on GitHub. 

That being said, the results clearly demonstrate that current models
are able to improve their performance on a well-posed game when given
the chance to provide new answers after seeing the results of their
previous choices.

\begin{figure}
\includegraphics[scale=0.4]{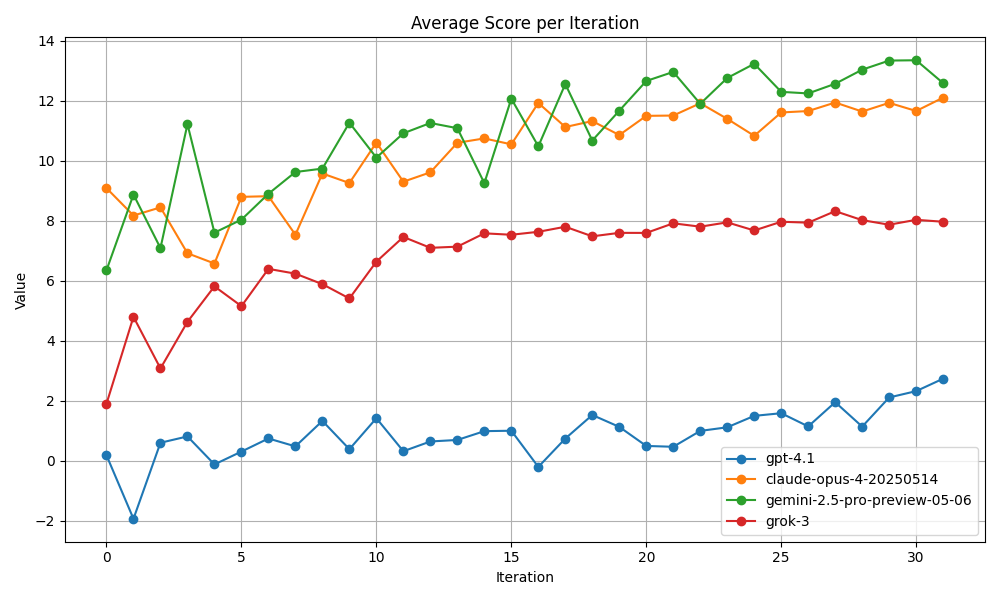}

\includegraphics[scale=0.4]{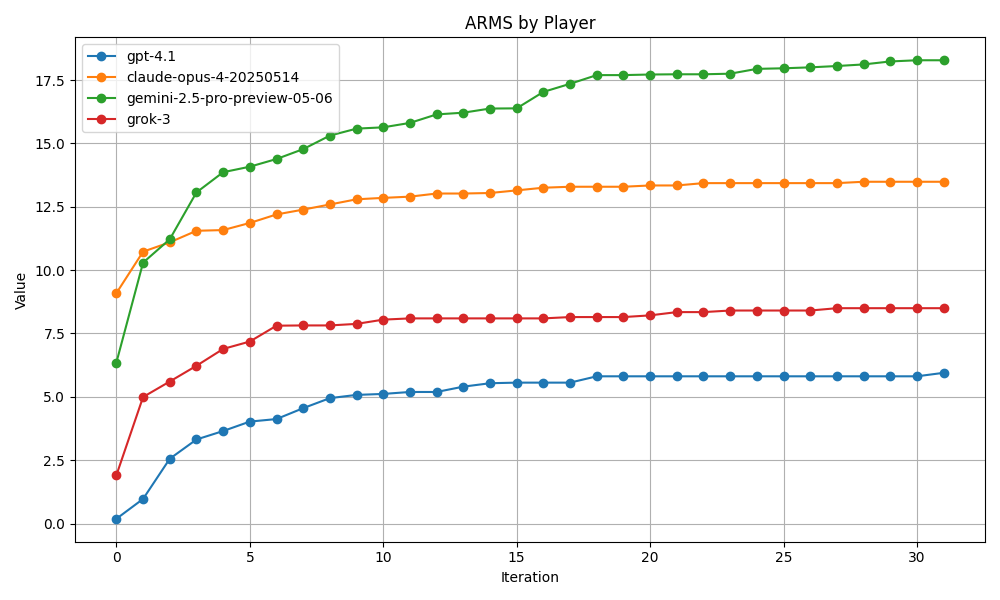}

\caption{\label{fig:performance-4-flagship-models}Performances for four flagship
models playing the single text Xent Game (with $N=20$, running for
$32$ iterations): the top plot shows the averaged scores (see Expression
\ref{eq:average-score}) and the bottom plot shows the $arms$ (see
Definition \ref{def:few-shot-playable-game}). We see that the $arms$
keeps improving for all models up to $k=17$. }
\end{figure}

\subsubsection{\label{subsec:fine-tuning-definition-playability}Fine-Tuning Definition}

The other, somehow deeper and more general, notion of playability
pertains to the ability of models to get better via fine-tuning. It
should be noted that closed models are not always available in fine-tuning
mode, and so this notion is strictly less applicable than the few-shot
playability one:
\begin{defn}
\label{def:fine-tuning-playble-game}A well-posed game $G$ is playable
in fine-tuning mode by a model $\mathcal{M}$ if, for some time $\mathcal{T}>0$,
the model $\mathcal{F}_{G}^{\mathcal{T}}\left(\mathcal{M}\right)$
is strictly better in expectation at $G$ than the model $\mathcal{M}$.
\end{defn}

\begin{rem}
It is generally expected that if a game is playable in fine-tuning
mode, it is also playable in few-shot mode.
\end{rem}

As we will see in Section \ref{subsec:transfer-value}, fine-tuning
playability is crucial for us to be able to relate games with one
another via the so-called transfer value. This notion of transfer
value allows us to build a geometric structure on the space of Xent
Games, which we rely upon to construct relevant Xent Measures. 

\subsection{\label{subsec:transfer-value}Transfer Value}

As discussed in Section \ref{subsec:playability-few-shot-and-fine-tuning-definitions},
a playable game is one where a player's strategy can be adjusted by
experience, leading to a higher score; in other words, the model $\mathcal{M}$
is \emph{learning} something. But is it \emph{worth} playing the game,
i.e. is the \emph{learned} information \emph{valuable}? Surely, this
is, in large part, very subjective. In this subsection, we will define
the value of learning to play a game in terms of the increased abilities
of an agent fine-tuned to play that game. This, of course, shifts
the subjectivity further: what is the end-goal in terms of game playing
utility? We defer the investigation of this latter problem to Section
\ref{sec:xent-game-measures}.

\subsubsection{\label{subsec:didactic-value-of-a-game}Transfer Value of a Game}

From the fine-tuning mechanism defined in Section \ref{subsec:fine-tuning-definition-playability},
we can now ask the question: how much does playing a game for a bit
improve an agent's skill on another game? This leads to define the
\emph{value} of a playable game:
\begin{defn}
Given a model $\mathcal{M}$ and a fine-tuning mechanism $\mathcal{F}$
and a fine-tuning time $\mathcal{T}>0$ (i.e. a token processing budget),
we define the\emph{ transfer value} (relative to $\mathcal{F)}$ of
a playable Xent Game $G_{1}$ for another playable game $G_{2}$ by
\[
\mathcal{V}_{G_{1}}\left(G_{2}\right)=\frac{S_{G_{2}}\left(\mathcal{F}_{G_{1}}^{\mathcal{T}}\left[\mathcal{M}\right]\right)-S_{G_{2}}\left(\mathcal{M}\right)}{S_{G_{1}}\left(\mathcal{F}_{G_{1}}^{\mathcal{T}}\left[\mathcal{M}\right]\right)-S_{G_{1}}\left(\mathcal{M}\right)},
\]
where $S_{G}$ denotes the expected score of a model when playing
$G$ (averaged over seeds).
\end{defn}

\begin{rem}
\label{rem:transfer-value}The transfer value should simply be understood
as `how much value does playing $G_{1}$ bring towards playing $G_{2}$?'
(normalized by how long one plays $G_{1}$). Note that there is no
reason to expect symmetry, i.e. $\mathcal{V}_{G_{1}}\left(G_{2}\right)$
is generally not equal to $\mathcal{V}_{G_{2}}\left(G_{1}\right)$. 
\end{rem}

\begin{rem}
As much as we will often take for granted that $\mathcal{M},\mathcal{F}$,
and $\mathcal{T}$ have been suitably set, working with transfer values,
it is important to remind ourselves that $\mathcal{V}$ depends a
priori on them, and that a `good setting' for $\mathcal{M},\mathcal{F},\mathcal{T}$
has been found. It is however important to keep in mind that we can
still expect that the \emph{final outputs} of our constructions (see
Sections \ref{sec:xent-game-measures} and \ref{sec:evolution-in-game-space}
below) to not depend on $\mathcal{M},\mathcal{F},\mathcal{T}$: we
expect that all `good' $\mathcal{M},\mathcal{F},\mathcal{T}$ settings
give roughly the same general ability measures (see Section \ref{subsec:universality-do-details-matter}
for a discussion of this universality idea).
\end{rem}

\begin{rem}
The idea of transfer learning across games appears in a number of
places in the literature, see e.g. \cite{banerjee-stone,parisotto-ba-salakhutidnov,textworld};
generally speaking it is reasonable to expect similar games to be
related. It is also interesting to note that nonzero transfer value
between an imperfect information game and its perfect information
counterpart (i.e. where all information is revealed) can be seen experimentally
(see \cite{bluem-czech-kersting}).
\end{rem}

While the transfer value of a game depends on the many details abstracted
in the $\mathcal{F}$ mechanism (and of the game metadata, in particular
the models linked to the judge and players), it is an interesting
question to ask how much the constructions we build from the transfer
value of a game depend on such details (see Section \ref{sec:evolution-in-game-space}). 

\subsubsection{\label{subsec:transfer-value-numerical-example}Transfer Value: Numerical
Example}

The transfer value concepts introduced above can be illustrated on
a simple triplet of games, shown in Figure \ref{fig:transfer-value-triplet}
(see also Section \ref{subsec:one-player-combinatorial-games} for
similar games).

\begin{figure}
\begin{lstlisting}
# Game Name: Single Text
assign(s=story())
elicit(t, 10)
ensure("no common words between"+s+"&"+t)
# asks for a nontrivial "summary" of s
reward(xed(s|t))
\end{lstlisting}

\begin{lstlisting}
# Game Name: Multi Texts
assign(s1=story(), s2=story(), s3=story())
elicit(t, 10)
ensure("no common words between"+s1+s2+s3+"&"+t))
# asks for a nontrivial "joint summary" of s1, s2, s3
reward(xed(s1|t)+xed(s2|t)+xed(s3|t))
\end{lstlisting}

\begin{lstlisting}
# Game Name: Dex Texts
assign(s1=story(), s2=story())
elicit(t, 10)
ensure("no common words between"+t+"&"+s1+s2)
# asks for a summary of s1, but anti-summary of s2
reward(xed(s1|t) + xed(s1|t) - xed(s2|t)) 
\end{lstlisting}

\caption{\label{fig:transfer-value-triplet}Triplet of one-player Xent Games
(played by $black$, in perfect information mode as per default)}
\end{figure}

These games, while clearly related, have distinct rules and strategies.
In particular, the third game utilizes a negative $xed$, meaning
that the playing agent is incentivized to find a prefix $t$ which
makes $s2$ unlikely.

Experimentally, we show the presence of transfer learning (i.e. nontrivial
transfer value) by performing three independent fine-tunings (one
on each of the three games) of a model (Qwen3 8B in FP8 in our experiment
with GPT2-XL as the judge model): for each of the three models being
fine-tuned, we evaluate their performance on all three games, as displayed
in Figure \ref{fig:evals-transfer-learning}. The fine-tuning procedure
$\mathcal{F}$ is performed using the REINFORCE algorithm \cite{williams}
on rewards associated with the choices of $\mathcal{M}$ at every
$elicit$ call, run with the AdamW optimizer and a learning rate of
$10^{-6}$ with a token processing budget $\mathcal{T}$, with a corpus
of 20k `stories'; the evaluation is performed on a corpus of 600 `stories'.
Our experiments demonstrate a clearly non-trivial transfer value across
games. Our results are obtained from the code available in the \href{https://github.com/xentlabs/xega/}{xega repository}
on GitHub. 

\begin{figure}
\includegraphics[scale=0.25]{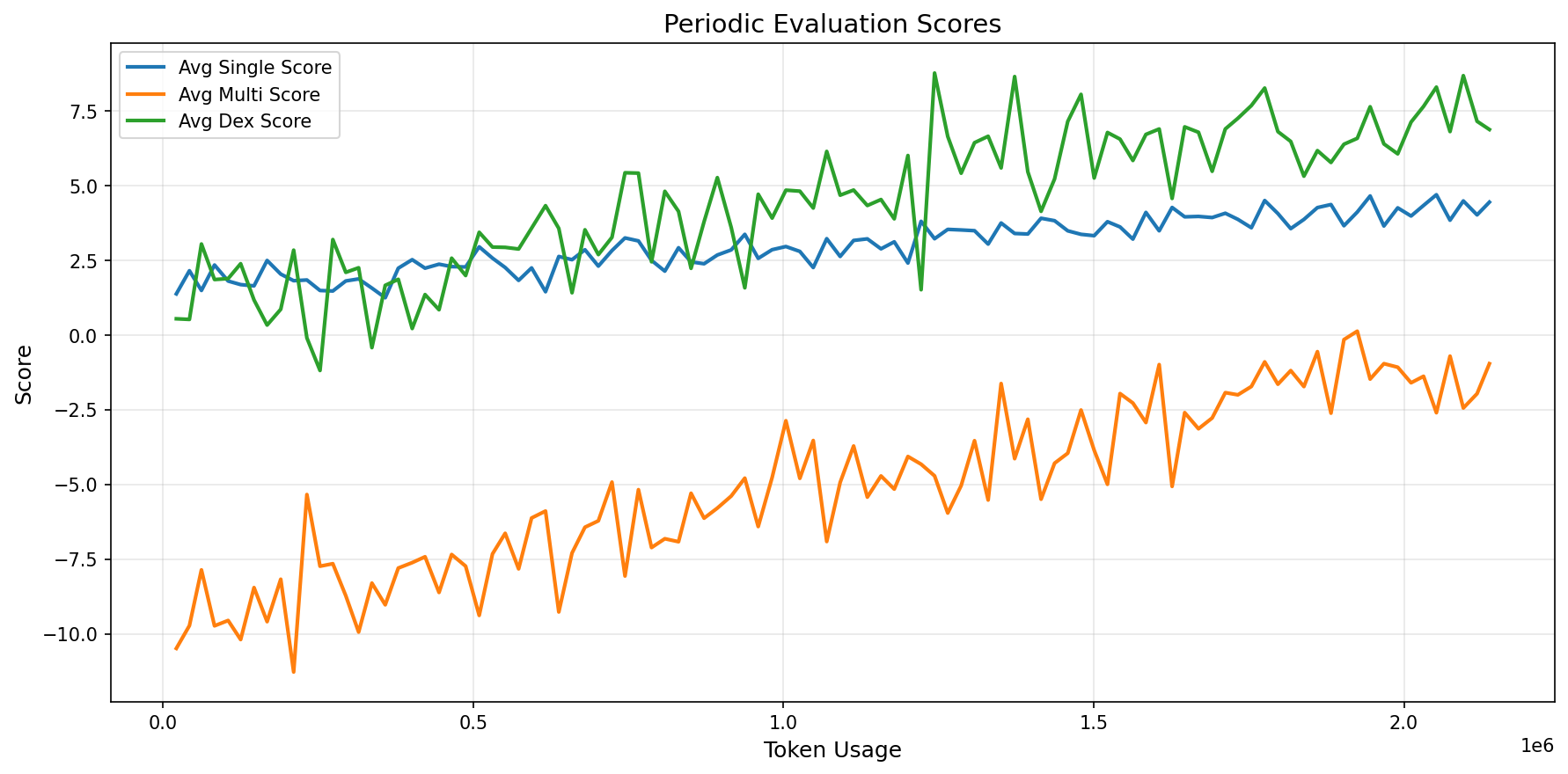}

\includegraphics[scale=0.25]{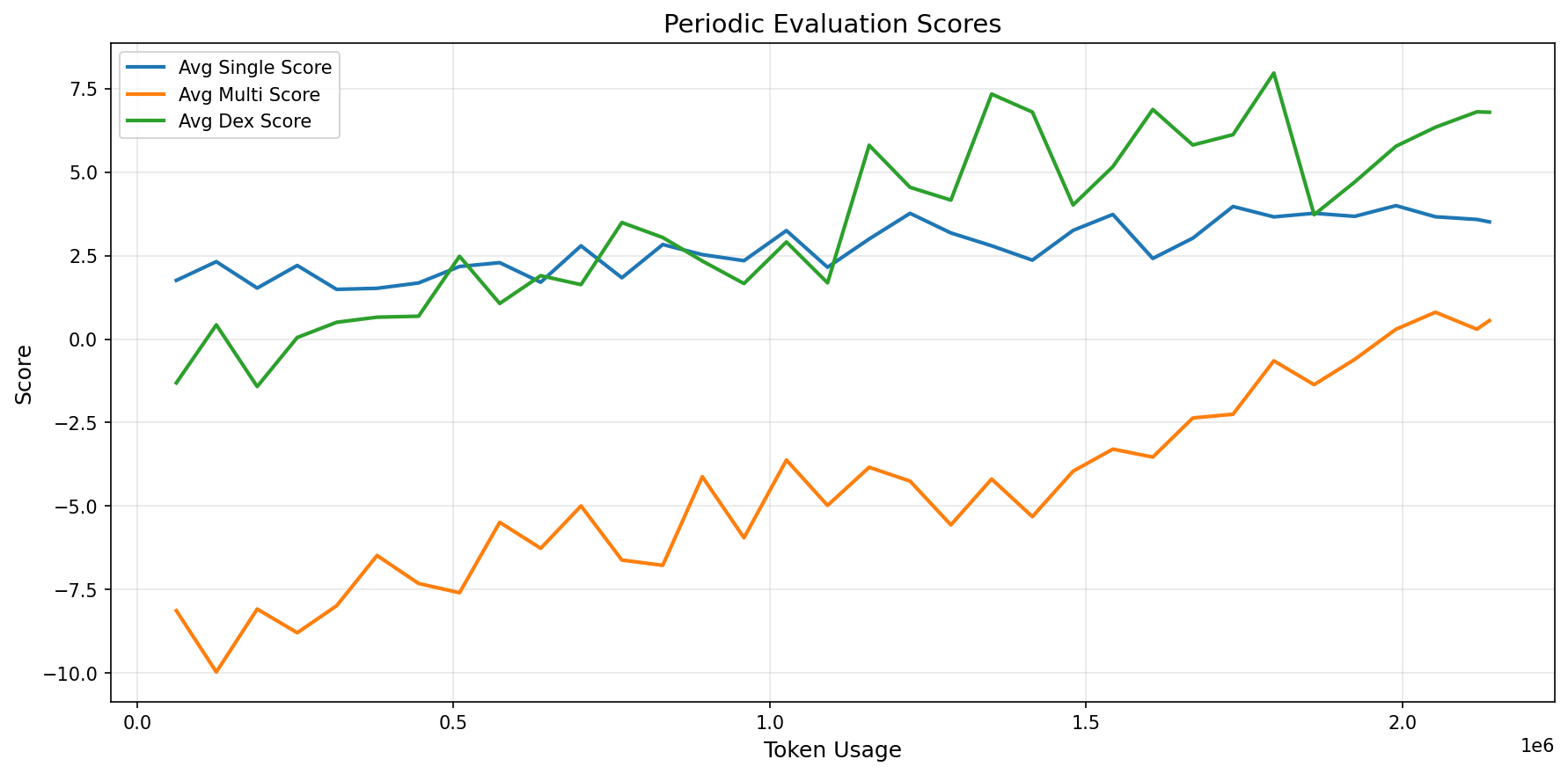}

\includegraphics[scale=0.25]{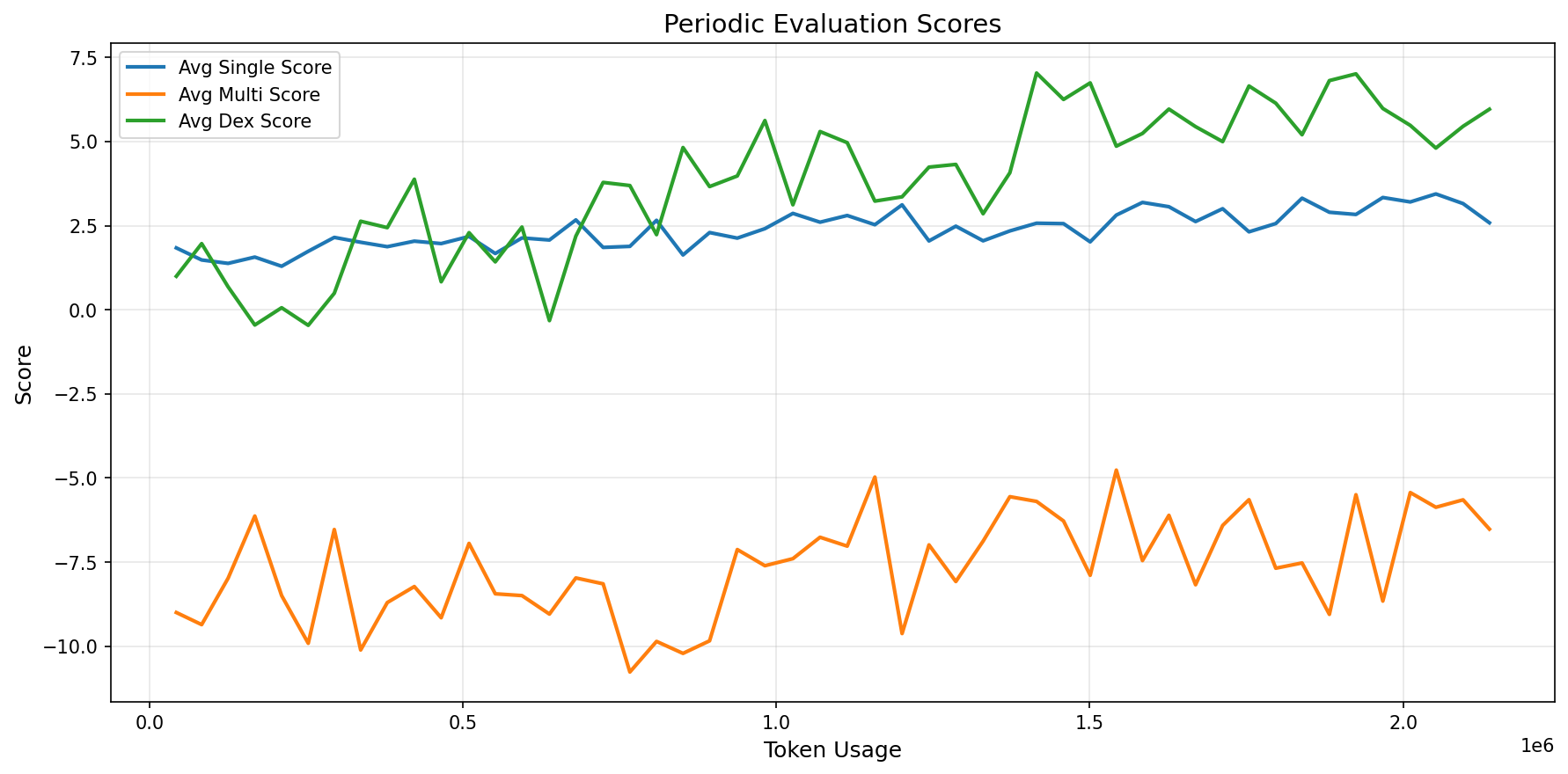}

\caption{\label{fig:evals-transfer-learning}Evaluation results during reinforcement
learning for \textquotedblleft Single Text\textquotedblright , \textquotedblleft Multi
Texts\textquotedblright , and \textquotedblleft Dex Texts\textquotedblright{}
games, respectively.}

\end{figure}
These results demonstrate that even in very simple configurations
there is clear transfer learning between games occurring and that
the transfer rates between games are varied and not symmetrical. We
see in particular that for the ``Dex Texts'' game, we have an asymmetry
between transfer values (see Remark \ref{rem:transfer-value}): for
instance, learning ``Single Text'' helps more for playing ``Dex
Texts'' than learning ``Dex Texts'' helps for playing ``Single
Text''. 

\section{\label{sec:xent-game-measures}Xent Game Measures: Scopes and Covers}

In this section, we construct \emph{Xent Measures}, which yield theoretically-motivated
LLM benchmarks from \emph{Xent Game scopes}, which are families of
playable Xent Games representing a desirable set of abilities. This
will pave the way towards Section \ref{sec:evolution-in-game-space},
which defines \emph{general capability scopes}, leading to general
ability measures.

\subsection{\label{subsec:benchmark-goals}Benchmark Goals}

In this subsection, we first review some works on LLM benchmarking,
in particular towards measuring general abilities. 

\subsubsection{From specific to general abilities: benchmarking}

Traditionally, the benchmarking of LLMs tests their explicit knowledge
using a set of questions (that are typically made by humans) with
e.g. multiple-choice answers. Designing such benchmarks is of course
very contingent on specific areas; there is also naturally the risk
of data contamination. See e.g. \cite{holistic,useful-llm-eval,survey-llm-eval}
for surveys on relevant work. 

Interesting examples of benchmarks that attempt to go towards more
general abilities include for instance \cite{chollet,chollet-knoop-kamradt-landers-pinkard}.
A very promising line of work that generally attempts at evaluating
LLM performance is via the creation of environments to evaluate them
in, see e.g. \cite{agent-gym,stojanovski-stanley-sharratt-jones-adefioye-kaddour-koepf,textworld},
though a principled way to extract a benchmark from such environment
is not yet available. The idea to use games has been popularized in
e.g. \cite{chatbot-arena}. 

In spite of there being a plethora of mechanisms to benchmark LLMs,
the interpretations and theoretical underpinnings of such benchmarks
are arguably rarely clear. In fact, one could even say that it is
not very clear what one expects from a benchmark, besides answering
the informal (and very underspecified) question of which LLMs are
`better and more capable'. Since these notions are not very well defined
a priori, it is useful to consider the question `what is expected
from a \emph{new} benchmark'. A key element one expects of a new benchmark
is for it to inform us along some dimension that is not captured accurately
by other benchmarks, but that one will recognize (perhaps a posteriori)
as being relevant to useful skills. 

While the practical successes yielded in part by the current state
of the art are undeniable (given that, ultimately, most are able to
recognize spectacular abilities of LLMs, without usually being able
to define them), this state of affairs makes it naturally hard to
ascribe a specific meaning to a benchmark score or to clearly identify
needs for future benchmark scores. 

\subsubsection{Vision and Goals}

The vision that we provide is in fact very simple: we would like benchmark
results to have simple interpretations in terms of concrete abilities
on (hopefully quite general) families of tasks, i.e. to provide us
with some statistics telling us `which fraction of tasks within a
certain scope can be accomplished at what skill level'. Such measures
can thus be used to compare capabilities of models with one another,
and also steer the development of LLM abilities to fill some gaps
in capabilities. In some sense, if we can find a fair sample of tasks
that are roughly equivalent in terms of relevance, the output benchmark
should be a histogram of the scores over the various tasks: should
the need to further summarize the model's abilities in one or several
scalars arise, means, medians, and more general moments and quantiles
can give a good idea of the situation. 
\begin{rem}
The idea to use text-based games as means to evaluate LLM abilities
has been considered in \cite{tan-kazemi-mihalcea}. 
\end{rem}

\subsection{\label{subsec:score-normalization}Score Normalization}

As discussed above, the goal of benchmarking is naturally to compare
various models (whether of different architectures, datasets, or simply
fine-tuning of a given model). Naturally, each playable instance of
a Xent Game $G$ can be used to measure various models; given the
diversity of such games, the question of score normalization, in particular
across games, roles, and seeds, is not obvious. 

\subsubsection{\label{subsec:base-model-and-few-shot-mode}Base Model and Few-Shot
Mode}

Related to the question of score normalization is that of units: what
are the units in which we aim to understand the performance of models
playing Xent Games? While Xent Games have the natural advantage of
having scores expressed in bits, surely the bits are not each of the
same value; even within the family of games that are playable for
some model, some are much harder to play than others. 

For instance, going from a 0.95 win rate playing chess against an
agent to a 0.999 win rate could be interpreted as a very drastic increase
in ability (much more than going from a rate of 0.5 to 0.55, say),
while it would correspond a small percent of score increase. 

In order to normalize scores, a simple solution is to identify a sufficiently
strong base model $\mathcal{M}_{\mathcal{B}}$, find reasonable values
of $N$ (number of samples),  $\epsilon>0$ (tolerated `failure probability'),
$\theta$ (increment value), and $m\geqslant2$ (number of increments),
and to consider the set of\emph{ $\left(N,\epsilon,\theta,m\right)$-playable
games} relative to $\mathcal{M}_{\mathcal{B}}$ (see Definition \ref{subsec:few-shot-definition-playability}),
which we will simply call (unless specified otherwise) the \emph{playable
Xent Games}. 

From there, the performance of any model $\mathcal{M}$ can be measured
in terms of \emph{how fast} $\mathcal{M}$ (in terms of iterations)
can reach the level of $\mathcal{M}_{\mathcal{B}}$, and hence yields
results that can be measured in terms of \emph{work} done per iteration
compared to $\mathcal{M}_{\mathcal{B}}$.

This measure of a model's abilities on a game in economic terms is
particularly useful to capture its performance on a task. If we have
a good estimate of a model's score, we will hence have a good estimate
on its \emph{progression curve} when playing a new game, and not only
on its eventual performance. 

\subsubsection{\label{subsec:normalized-score-definition}Normalized Score Definition}

As suggested in Section \ref{subsec:base-model-and-few-shot-mode},
we leverage the concept of playability to obtain a unified normalization
scheme for Xent Games, once a base model $\mathcal{M}_{\mathcal{B}}$
is picked. 
\begin{defn}
Let $G$ be \emph{$\left(N,\epsilon,\theta,m\right)$}-playable by
a base model $\mathcal{M}_{\mathcal{B}}$ for some fixed values of
\emph{$\left(N,\epsilon,\theta,m\right)$} (see Section \ref{subsec:few-shot-definition-playability}
above), with base mean score $b$ at the first iteration times $k_{j}$
when $j$ exceeds $b+j\alpha$ for $j=1,\ldots,m$. Then for any $\ell\leqslant k_{m}$
and any model $\mathcal{M}$ yielding $arms_{N,k}$ scores (see Section
\ref{subsec:few-shot-definition-playability} above) at the various
iterations, the normalized score of $\mathcal{M}$, denoted $narms_{N}$,
at time $\tau=\sup\left\{ k\in\left\{ 1,\ldots,\ell\right\} :arms_{N,k}\leq b+m\alpha\right\} $
is defined by:
\end{defn}

\begin{itemize}
\item If $\tau=\ell$ (i.e. $\mathcal{M}$ has not outpaced $\mathcal{M}_{\mathcal{B}}$
in $\ell$ steps) and $arms_{N,\tau}=b+j\alpha$ for some $j\leq m$,
then $narms_{N}$ is exactly $j$; anything in between is the unique
affine interpolation that is continuous in $arms_{N,\tau}$. 
\item If $1<\tau<\ell$ (i.e. $\mathcal{M}$ has outpaced $\mathcal{M}_{\mathcal{B}}$
in $\ell$ steps) and $arms_{N,\tau}$ is exactly $b+m\alpha$, the
$arms_{N}$ score is $m\frac{\ell}{\tau}$; anything in between is
given by the unique affine interpolation that is continuous in $arms_{N,\tau}$. 
\end{itemize}
Informally, the normalized score consists of the number of play steps
of $\mathcal{M}_{\mathcal{B}}$ that $\mathcal{M}$ is able to catch
up with per step. 

\subsubsection{\label{subsec:multi-game-benchmark-design}Multi-Game Benchmark Design}

In Section \ref{subsec:score-normalization}, we presented the question
that a unified normalization aims to answer: What is the performance
of a model playing a game, relative to a base model?

Following this route, we can now discuss the question of a multi-game
benchmark normalization. Again, a fairly simple question can be asked:
if a model is supposed to play a large number of games, if we run
it for a certain number of steps on each game, what does it deliver
on average, per game, relative to a base model $\mathcal{M}_{\mathcal{B}}$?
The idea of using an averaging measure again follows from economic
considerations: running costs are naturally additive. In addition
to the above motivation, averages (or perhaps more generally weighted
averages) have the advantage of being \emph{maximally sensitive} to
individual game benchmarks (if we use weighted averages, we can explicitly
say how important each game is). 

The problem of \emph{selection} of the games appears, however, to
be substantially more subjective and challenging: we discuss this
question in Section \ref{sec:xent-game-measures}, and an answer based
on an evolution-inspired meta-game will be provided in Section \ref{sec:evolution-in-game-space}. 

\subsection{\label{subsec:benchmarks-from-xent-measures}Benchmarks from Xent
Measures}

As a way to propose a solution to the benchmarking problem raised
in Section \ref{subsec:benchmark-goals} above, it is tempting to
use the structure of the Xent Game space \ref{subsec:characterization-result}
and suggest reliance on benchmarks of the following forms for a suitable
(finite) family of games $\Gamma$:
\begin{defn}
For a family of games $\Gamma$ and a normalized score function $S^{*}$
(such as the $narms$ defined in Section \ref{subsec:base-model-and-few-shot-mode}),
the Xent (probability) measure $\mu_{\Gamma}\left(\mathcal{M}\right)$,
is defined as $\frac{1}{\left|\Gamma\right|}\sum_{G\in\Gamma}\delta_{S^{*}\left(\mathcal{M}\right)},$
where $\delta_{x}$ denotes the Dirac mass at $x\in\mathbb{R}$. 
\end{defn}

\begin{rem}
The above should simply be understood as the histogram of the set
of values $S^{*}\left(\mathcal{M}\right)$ over $G\in\Gamma$ (accounting
for repeated values). 
\end{rem}

In order for $\mu_{\Gamma}$ to serve as the basis of a `good' benchmark
for a given scope, the games in $\Gamma$ must hence be chosen in
such a way that:
\begin{itemize}
\item The family is large enough to satisfactorily cover the diversity of
the scope. 
\item The covering is `fair', i.e. there is no `over-representation' of
some games in the scope. 
\end{itemize}
In Section \ref{subsec:evolution-based-scope-expansion} below, we
will see that the notion of transfer value defined above (Section
\ref{subsec:transfer-value}) can be used to achieve such goals. 

\subsection{\label{subsec:xent-measure-from-scope}Xent Measure from Scope}

The idea behind the definition of Xent Measures from scopes is simple
and very related to meta-learning (see e.g. \cite{baxter-inductive-learning}
for foundations, \cite{meta-learning-llm} for a recent survey about
LLMs, and \cite{tutorial-on-meta-rl} for a survey about meta-reinforcement
learning): given a \emph{scope} $\Sigma$, find a minimal covering
family $\Gamma$ such that any game in $G\in\Sigma$ is `close enough'
to a game in $G'\in\Gamma$ (in the sense that $\mathcal{V}_{G'}\left(G\right)$
is large enough). In other words, we want a minimal set of games $\Gamma$
such that if we had to play any game $G\in\Sigma$, and if we are
able to play the games in $\Gamma$, we could quickly fine-tune our
model to play $G$. 

A minimal covering family $\Gamma$ naturally yields a capability
benchmark for LLMs on the scope $\Sigma$: for a model $\mathcal{M}$,
consider the histogram of the scores $S_{G}\left(\mathcal{M}\right)$
for $G\in\Gamma$. Intuitively, the idea is that if $\Sigma$ is dense,
each $G\in\Gamma$ plays a similar role, and can thus be considered
interchangeably, with the minimality guaranteeing `no over-representation'
of certain regions of the game space: the histogram of the scores
then gives us a good idea of how well $\mathcal{M}$ can be expected
to perform on a random game sampled `near $\Sigma$' (i.e. in a region
well covered by $\Sigma$). 

Despite the fact that the following definition is quite theoretical
(even for very simple cases, computing the optimal cover is NP-hard),
it is reasonable to aim for approximations (and find good covers).
\begin{defn}
Consider a scope $\Sigma$, and a space of games $\Upsilon$.
\begin{itemize}
\item For a family $\Gamma\subset\Upsilon$, we denote by $\Delta_{\Gamma}\left(\Sigma\right)$
the \emph{transfer cover} of $\Sigma$, defined as $\min_{G\in\Sigma}\max_{G'\in\Gamma}\mathcal{V}_{G'}\left(G\right)$:
this represents how much $\Sigma$ is guaranteed to be covered by
$\Gamma$.
\item We call a family of Xent Games $\Gamma$ an $\eta$-cover if $\Delta_{\Gamma}\left(\Sigma\right)\geq\eta$.
\item We say that an $\eta$-cover of $\Sigma$ is optimal if its cardinality
is minimal and if its transfer cover is maximal for its cardinality,
and we denote such a cover $\Gamma_{\Sigma}^{\eta}$.
\end{itemize}
\end{defn}

\begin{rem}
We have that $\Gamma_{\Sigma}^{\eta}$ is, in some sense, a `skeleton'
of $\Sigma$, i.e. a version of $\Sigma$ stripped of duplicate games,
and possibly games replaced by `leaner' versions of themselves, for
games $G$ such that there exists a $G'$ with $\mathcal{V}_{G'}\left(G\right)>\mathcal{V}_{G}^ {}\left(G\right)$.
\end{rem}

From the above, we obtain: 
\begin{defn}
For a scope $\Sigma$, and a cover $\eta$ we call the benchmark measure
associated with $\left(\Sigma,\eta\right)$ the Xent Measure $\mu_{\Sigma}^{\eta}:=\mu_{\Gamma_{\Sigma}^{\eta}}$.
\end{defn}

\begin{rem}
Even if the scope $\Sigma$ is large, most of the interesting covers
should correspond to relatively small values of $\eta$: these correspond
to fairly coarse covers. 
\end{rem}

\subsection{\label{subsec:scope-selection}Scope Selection}

From Section \ref{subsec:xent-measure-from-scope}, given a finite
scope $\Sigma$, a Xent benchmark measure $\mu_{\Sigma}^{\eta}$ can
be defined for any $\eta\leq1$. This leaves open the question of
choosing the scope itself. Depending on one's goals, different types
of considerations can apply: 
\begin{enumerate}
\item \label{enu:specific-goal}We have a set of specific goals that are
covered by a set of games, such as the ones introduced in Section
\ref{sec:xent-game-examples}.
\item \label{enu:discriminative-goal}We want to highlight differences in
capabilities between various models: we would like to select games
that are maximally discriminative between models. 
\item \label{enu:general-goal}We are aiming to measure general abilities,
in particular how a model will perform in a setting that is not known
a priori.
\end{enumerate}
These considerations can apply separately or simultaneously: for instance,
we may have a set of pre-defined goals, but would like to expand in
generality around those, or may want to set specific goals to allow
a model to differentiate itself from other models. 

Since the key problem raised by the Items \ref{enu:specific-goal}
and \ref{enu:discriminative-goal} is to construct a specific list
of games (whether hand-crafted or picked according to a feature selection
process), if we want to complete the above picture with the generality
dimension of Item \ref{enu:general-goal}, the following question
emerges: how to grow, starting from a list of specific games, a sequence
of games that measure more and more general capabilities around that
initial list? This is the content of Problem \ref{prob:scope-growth-question}
below. 

\subsubsection{\label{subsec:towards-a-general-ability-measure}Towards a General
Ability Measure}

A naive approach would be to say that since `general abilities' are
about all the games, we should want to put all games on equal footing;
an obvious problem is the infinite size of the resulting scope, making
it unlikely we can cover the space with a finite (let alone small)
set of games. Intuitively, this makes sense, as it is unlikely that
`general abilities' could ever be well summarized by a fixed number
of games: for any finite game set, it is very plausible we can find
a game that is substantially different from all games in that set. 

Giving up on the idea of a finite set of games appears to be a promising
direction, bringing us closer to something like a curriculum. Naturally,
this forgoes the idea of a uniform benchmark measure (see Remark \ref{rem:infinite-scope-subjectivity}
below), and leads us to focus on an \emph{order} for games that are
to make up a benchmark. 
\begin{problem}
\label{prob:scope-growth-question}Starting from an initial scope
$\Sigma_{0}$, what is a principled way to define a sequence of Xent
Measures $\mu_{k}$ which form a good estimate of a model's general
abilities around $\Sigma_{0}$?
\end{problem}

\begin{rem}
\label{rem:infinite-scope-subjectivity}While ordering measures (and
scopes) introduces a subjective bias, this seems absolutely necessary
for both theoretical and practical reasons: this is similar to the
fact that it is impossible to sample an (unbounded) integer uniformly
at random. 
\end{rem}

\begin{rem}
\label{rem:general-measure-as-limit}In practice, if, for a given
model $\mathcal{M}$, the Xent Measures $\mu_{k}\left(\mathcal{M}\right)$
stabilize around some value $\mu_{*}$, it seems reasonable to ascribe
that value $\mu_{*}$ as a measure of $\mathcal{M}$ as a `general
ability measure of $\mathcal{M}$' (even though the latter may not
be definable generally: this is the same kind of criterion used to
say that `50\% of the integers are even', in spite of the fact that
there is no uniform measure on integers).
\end{rem}

While this question is naturally very under-specified, we will see
in Section \ref{sec:evolution-in-game-space} that ideas inspired
by evolution ideas can lead us towards principled answers to Problem
\ref{prob:scope-growth-question}. 

\section{\label{sec:evolution-in-game-space}Evolution in Game Space}

In Section \ref{sec:xent-game-measures}, we introduced the idea of
benchmarks based on \emph{scopes}, i.e. families of Xent Games. In
Section \ref{subsec:towards-a-general-ability-measure}, we discussed
the challenges associated with infinite scores, and formulated Problem
\ref{prob:scope-growth-question}: how to find, starting from an initial
family of games $\Sigma_{0}$, a sequence of measures that can be
viewed (thinking of the $k\to\infty$ limit) as an approximation of
general abilities (in the sense of Remark \ref{rem:general-measure-as-limit})? 

In this section, we consider ways to explore the space of Xent Games
by leveraging its geometric structure. We then propose, from game-theoretic
considerations and evolution-inspired ideas, mechanisms to define
sequences of measures $\mu_{k}$ that aim at measuring general model
capabilities. 

\subsection{\label{subsec:xent-game-space-geometry}Xent Game Space Geometry}

One of the key insights about the Xent Game space is that the games
are related to one another via their structures: this is essentially
a consequence of the proof of Theorem \ref{thm:characterization-perfect-information-xent-games}
above. This suggests (informally) that if one has a strong enough
base model $\mathcal{M}_{\mathcal{B}}$, one should be able to transform
any (reasonable) Xent Game into any other Xent Game by a sequence
of moves such that we can transfer some non-trivial amount of skill
from game to game (locally, not necessarily across the whole path). 
\begin{defn}
For $\alpha>0$ and a base model $\mathcal{M}_{\mathcal{B}}$, we
write $G_{1}\rightmoon_{\alpha}G_{2}$ if $\mathcal{V}_{G_{1}}\left(G_{2}\right)\geqslant\alpha$
and $G_{1}\leftmoon_{\alpha}G_{2}$ if $\mathcal{V}_{G_{2}}\left(G_{1}\right)\geqslant\alpha$,
and call a finite sequence of games $G_{1},\ldots,G_{n}$ an $\alpha$\emph{-downstream
path} if $G_{1}\rightmoon_{a}\cdots\rightmoon_{\alpha}G_{n}$ and
an $\alpha$\emph{-upstream path} if $G_{1}\leftmoon_{\alpha}\cdots\leftmoon_{\alpha}G_{n}$. 
\end{defn}

\begin{rem}
As suggested above, there is no reason to expect a priori $\rightmoon_{\alpha}$
to be a transitive notion. 
\end{rem}

\begin{defn}
For $\alpha>0$ and a collection of games $\Sigma_{0}$, we denote
by $\left(\Sigma_{0}\right)_{\alpha}$ the \emph{$\alpha$-downstream
component} defined as the set of games $G$ reached by an $\alpha$-downstream
path that starts from a game in $\Sigma_{0}$; similarly, we denote
by $\left(\Sigma_{0}\right)^{\alpha}$ the \emph{$\alpha$-upstream}
\emph{component} defined as the set of games reached by an $\alpha$-upstream
path starting from a game in $\Sigma_{0}$.

\label{def:alpha-connected}We say that a space of Xent Games is $\alpha$-connected
if there is a finite $\Sigma_{0}$ such that the downstream component
$\left(\Sigma_{0}\right)_{\alpha}$ is the whole space. In other words,
the set of Xent Games is $\alpha$-downstream connected to a finite
set of games.

For a `sophisticated enough model', it is reasonable (by the construction
of Theorem \ref{thm:characterization-perfect-information-xent-games}
above) to postulate that all the (playable) Xent Games are related
in the following sense: one can always modify a playable game into
another one by a chain that is $\alpha$-downstream and by a chain
that is $\alpha$-upstream, for some small enough $\alpha$. This
is the content of the following: 
\end{defn}

\begin{assumption}
\label{assu:alpha-connectedness}There exists some $\alpha>0$ such
that the space of playable Xent Games is $\alpha$-connected, i.e.
that we can reach any playable game from a finite set of basic playable
games (e.g. the basic Xent Game) by $\alpha$-downstream paths and
also by $\alpha$-upstream paths. 
\end{assumption}

\begin{rem}
This may suggest to take $\alpha$ small enough to cover the whole
space of games (note that if $\alpha$ is too small, games will only
be marginally related); or we could also instead decide to focus on
an $\alpha$-connected component for some a priori fixed $\Sigma_{0}$.
Generally speaking, the choice of $\alpha$ seems to be a subtle question
which we plan on investigating in further works. 
\end{rem}

As the space of Xent Games is formally infinite, the situation may
seem analogous to the situation where we need to explore an infinite-dimensional
vector space, and it may as a result seem hopeless to `exhaust' it
\emph{a priori }(i.e. without additional structures) with a growing
sequence of `uniform' measures on finite sets (as follows from e.g.
Baire's Theorem). As explained in Section \ref{subsec:web-exploration-intuition}
below, we will essentially give up on this analytic point of view
to focus on heuristics coming from samples from exploration processes
that leverage some weak form of ergodicity: if the measures we find
through different runs are consistent for a given model, we postulate
that we have a `good' measure of the model's ability (and we postulate
that inconsistencies are unlikely, based on universality assumptions,
see Section \ref{subsec:universality-do-details-matter}). 
\begin{rem}
Once a base model $\mathcal{M}_{\mathcal{B}}$ is fixed, the geometric
picture outlined above is most of what we use about playable Xent
Games in this section (i.e. we abstract the specific details of Xent
Game implementation). 
\end{rem}

\subsubsection{\label{subsec:web-exploration-intuition}Web Exploration Intuition}

If we assume that the space of Xent Games is $\alpha$-connected for
some $\alpha>0$, we can think of it as an (infinite) graph and we
can put a downstream link from $G_{1}$ to $G_{2}$ if $G_{1}\rightmoon_{\alpha}G_{2}$
and an upstream link if $G_{1}\leftmoon_{\alpha}G_{2}$. We can hence
leverage this structure to think of the space of games as a kind of
`inter-game web' and use this vision as a guiding principle to explore
it like a (very) complex network. 

Note that unlike the internet case (where listing all the outgoing
links from a page is trivial), it is a priori very non-trivial to
list all the upstream and downstream links from a given game; however,
we can think that in practice, if we have a finite list of games $\Sigma$,
it is definitely possible to determine which ones are $\alpha$-linked
and which ones are not, thus equipping $\Sigma$ with that structure.
\begin{rem}
In practice, the way the exploration around a node $G$ is performed
should be by asking an LLM to perform various `mutations' around $G$,
and finding the $\alpha$-links among those. 
\end{rem}

Following this principle, it is tempting to think of the exploration
of the space of games as being similar to `random web browsing' (following
the philosophy of the PageRank algorithm \cite{page-brin-motwani-winograd}).
If in some sense, one wants to find a `good sample' of the internet,
one can browse randomly and save the visited pages on the fly. Depending
on the application, we may even move for long enough so that our visit
becomes independent from the starting point (note also that the longer
we `walk', the more biased towards the more connected nodes we get).
From this, we could, in principle, use various heuristics (e.g. the
PageRank algorithm) to rank the games that we care about for a target
application. 

In order to implement the above general idea of `exploration', we
still need to clarify what `application' we aim for (at least, implicitly),
in particular we try to estimate general capabilities of models. In
Section \ref{subsec:evolution-based-scope-expansion} below, we propose
a principled way to perform this, based on evolutionary ideas. 

\subsubsection{\label{subsec:toy-model}Toy Model}

As a means to gain some intuition on the geometry of the Xent Game
space, we can use the following model: 
\begin{itemize}
\item Represent by indices $1,2,3,\ldots,n,\ldots$ an (infinite) list of
independent `skills' that may be associated with games.
\item Model each game $G$ by an infinite nonzero vector $\left(g_{1},g_{2},\ldots,g_{n},\ldots\right)$
of nonnegative entries that have (uniformly) bounded sum $\|g\|_{1}$,
with each entry $g_{i}$ representing naively `how much the skill
$i$ is involved (and learnable) when playing $g$'. 
\item For any pair of games $G,\tilde{G}$, corresponding to $\left(g_{1},g_{2},\ldots,g_{n},\ldots\right)$
and $\left(\tilde{g}_{1},\tilde{g}_{2},\ldots,\tilde{g}_{n},\ldots\right)$,
we define $\mathcal{V}_{G}\left(\tilde{G}\right)$ as $\frac{1}{\|\tilde{G}\|_{1}}\sum_{i=1}^{\infty}\min\left(\tilde{g}_{i},g_{i}\right)$,
i.e. which fraction of $\tilde{g}$ skills are `captured' by playing
$g$ and, symmetrically, we define $\mathcal{V}_{\tilde{G}}\left(G\right)$
as $\frac{1}{\|G\|_{1}}\sum_{i=1}^{\infty}\min\left(g_{i},\tilde{g}_{i}\right)$.
\end{itemize}
\begin{rem}
This simply suggests the (naive, but useful) view that games are associated
with skills that are entries in a vector, and that learning these
games amounts to acquiring skills proportionally to the vector of
skills. 
\end{rem}

As we will see in Section \ref{subsec:numerical-simulation-exploration-in-toy-model},
taking random games (by picking a distribution for the skill assignments
for games), we can construct a model space to investigate various
exploration algorithms numerically. 

\subsection{\label{subsec:evolution-based-scope-expansion}Evolution-Based Scope
Expansion }

\subsubsection{\label{subsec:meta-game-informal-dynamics-idea}Meta-Game: Informal
Evolution Dynamics Idea}

In a sense, the question of measuring `general capabilities' of models
comes down to a motivation: why do we want to evaluate general abilities
in the first place? A plausible answer is because agents will be exposed
to unforeseen problems; but then it is important to understand what
could possibly generate these unforeseen problems. If we think of
challenges that intelligent (animal) agents face in the real world
(especially unforeseen ones), they very often come down to the fact
that other agents are present (animals are in competition for resources,
can be preys or predators, must reproduce, etc.). 

These kinds of questions and ideas have been explored in the open-ended
evolution context, in particular under the context of \emph{Quality-Diversity}
search (see e.g. \cite{lehmann-stanley-evolution-novelty,lehmann-stanley-novelty-search,clune,etcheverry-chan-moulin-frier-oudeyer}).

Based on such ideas, we propose to explore the space of Xent Games
using a competitive vision which is close in spirit to the Novelty
Search with Local Competition (NSLC) algorithms proposed in \cite{lehmann-stanley-evolution-novelty,lehmann-stanley-novelty-search}.
We can imagine a world where a number $N$ agents play a (repeated)
meta-game where \emph{each move is the generation of Xent Games}.
The meta-game design elements consist of roughly the following:
\begin{itemize}
\item Agents can choose some games they want to specialize in (i.e. that
they decide to learn playing `privately'), which they will play with
a random sample $k$ of other agents (this would naively correspond
to the `field' on which they choose to live) in zero-sum mode.
\item Agents must at the same time play against the $k$ randomly sampled
agents (again, in zero-sum mode) with the games picked by the latter
(on which they will have specialized, by symmetry).
\end{itemize}
If we think of the competitive landscape, agents will hence be incentivized
to pick specialization in games that give them an edge over other
agents' games (because of the zero-sum nature of the games: these
are resource allocation games), while at the same time picking games
at which other agents' skills are under-developed. This naturally
suggests a competitive dynamic will take place, which should see the
games evolve over time. 
\begin{rem}
This view of an evolving landscape generated by agents in a local
game (or meta-game here) is for instance studied in \cite{laland-et-al}.
The idea of using evolution-related ideas for artificial general intelligence
has seen some uptick recently, see in particular \cite{aa-team,hughes-dennis-parker-holder-bekhabani-mavalankar-shi-schaul-rocktaechel,zhang-hu-lu-lange-clune}.
Note that these ideas are also similar to the ideas of artificial
curiosity and self-improvement \cite{schmidhuber,schmidhuber-survey,clune}. 
\end{rem}

Leveraging such a dynamic as a `browsing mechanism' (to use the analogy
of Section \ref{subsec:web-exploration-intuition}) to build a \emph{game
archive} (the set of games visited) used to measure general abilities
is appealing (as far as building a scope is concerned), as it provides
a general idea of where `unforeseen new games' could plausibly come
from; we could therefore use the list of visited games as a basis
for a benchmark. 
\begin{rem}
While it can definitely be interesting to specify the details of the
above meta-game, in the current paper we only use it as a means to
inform the selection of relevant games to build appropriate Xent Measures.
As such, it reasonable (assuming universality, see Section \ref{subsec:universality-do-details-matter}
below) to only assume that agents all use the same fine-tuned variants
of the base model $\mathcal{M}_{\mathcal{B}}$. 
\end{rem}

\subsubsection{\label{subsec:key-features-of-the-meta-game}Key Features and Key
Challenges of the Meta-Game}

The vision outlined in Section \ref{subsec:meta-game-informal-dynamics-idea}
above promises a number of interesting features:
\begin{itemize}
\item If we start from a set of `core abilities' represented by an initial
game $\Sigma_{0}$, we have a good idea of `how far we drifted' from
$\Sigma_{0}$ as the process runs, and we can in principle use this
to adjust the weights of the new games accordingly.
\item There is fairly clear idea of what we mean by `general capabilities'
in this context: the ability to make up for a lack of a specialization,
that could realistically be picked and indeed learned by another agent
(i.e. for a `useful' reason). In other words, agents that are `generally
more capable' will withstand confrontation with other agents on their
respective fields, while keeping an edge on their own. 
\item It is also consistent with our view on game normalization, where scores
represent the ability to make up for a number of attempts: in a sense,
we measure the ability to make up for a missed `virtual number of
attempts' that could have been emulated by an alternative specialization. 
\item If $N$ is not too large, the dynamics will naturally avoid oversampling:
agents are incentivized to find `original' directions that are far
away from existing games, otherwise they will get `eaten' by competitors;
if there is an `oversampled region', it naturally becomes prey for
`original' agents.
\item If $N$ is not too small, the agents will at the same time be incentivized
to stay `generally relevant', i.e. not to `drift towards isolated
niches' (that would deprive them from an ability to compete with the
others). 
\item This suggests a view about general abilities formed by the aggregate
of metastable equilibria: driven by the individual `initiatives' that
are selected on the basis of being relevant between covering existing
abilities and uncovering new directions. Hence, from the meta-game
dynamics, the only for agents to not extend to broader (i.e. `uncaptured')
regions of the game space is if there are no paths of `meaningful'
games to them. 
\end{itemize}
In spite of the appealing features listed above, a number of complex
features associated with the meta-game described in Section \ref{subsec:meta-game-informal-dynamics-idea}
make it very under-specified and challenging to implement in practice:
\begin{itemize}
\item Many unspecified details about the game are, in fact, difficult to
set: how many players are involved throughout the steps, how are they
different from one another, how much can they capitalize upon their
previous choices, how are they making their choices, what information
do they have about each other, etc? 
\item The dynamics suggest an `intersubjective' formulation of general capabilities:
general capabilities move in some sense towards a direction resulting
from what agents expect general abilities to look like. It can lead
to a very unstable set of games that `each believes the other believes,
but no one truly believes it'.
\item It is not clear how things will behave if the number of players becomes
large, and whether determining an optimal strategy is even feasible;
if the game is a partial information game, the corresponding Nash
equilibria are likely to be impossible to approximate in practice. 
\item While this is not necessarily a drawback, the evolution should involve
randomness: in natural settings, it is pretty clear that the unpredictable
component of evolution is an integral part of it and that no two `runs'
of a system will yield the same results. 
\item It is not clear exactly how players are supposed to play in practice:
should they communicate to cooperate, do they have bounded rationality,
should they assume the other players are being played by some fixed
models, etc.?
\end{itemize}
It should be obvious that determining a precise, definitive answer
to these questions, if one tries to model real-world-inspired dynamics,
is impossible. Ultimately, a way out is possible if the result of
the dynamics, as far as measuring general capabilities is concerned,
does not depend on these details: this is the view proposed in Section
\ref{subsec:universality-do-details-matter}. 

\subsubsection{\label{subsec:universality-do-details-matter}Universality: Do Exploration
Details Matter?}

A way out of the difficulties associated with playing the meta-game
discussed above is to ask the question: does the precise meta-game
play algorithm really matter? More precisely: if all we care about
is selecting lists of Xent Games to benchmark the general abilities
of agents, does the specific choice of games picked by the agents
in the meta-game actually matter?

Heuristically, there is a very simple reason why specific details
should not matter: the whole point of any good measure of general
abilities is that it should \emph{actually be fairly resilient to
changes of the rules of the games we use}. Agents with strong general
capabilities will keep their strong performances on games that are
known to be connected, while over-specialized agents are unlikely
to get systematically lucky over a large, diverse enough family of
games. Going further, any strong agent playing the meta-game should
be playing in an environment where it can withstand any \emph{specific,
computable} choice of games: any such choice of games is in itself
a possible play strategy, and any meta-game agent should choose games
so that it can be resilient against such a specific strategy. Similarly,
as suggested in Section \ref{subsec:didactic-value-of-a-game}, we
expect the dependence of the measures of general abilities to be fairly
independent of the precise setting (the $\mathcal{M},\mathcal{F},\mathcal{T}$
setting) associated with the transfer value functions (as long as
they are `good'). 

It is, in fact, reasonable to postulate \emph{the only relevant features
are that the exploration covers a (1) connected and (2) diverse enough
web of games}, \emph{sampled in a (3) fair way}. While the connectedness
is actually easy to guarantee (by construction), it is important to
make sure that the scope does not miss crucial games, i.e. games that
are important to learn other games. As we will see in Section \ref{subsec:scope-growth}
below, a criterion that essentially implies this is the following:
\begin{assumption}
\label{assu:boundedness}There exists a number $\alpha>0$ and a number
$K\geq1$ such that the space of Xent Games is $\alpha$-connected
and such that for any game $G$, the maximal number $n$ of `escaping
directions', i.e. of games $G_{1},\ldots,G_{n}$ such that 
\begin{align*}
\min_{k=1,\ldots,n}\mathcal{V}_{G_{k}}\left(G\right) & \geqslant\alpha\\
\max_{j=1,\ldots,k-1}\mathcal{V}_{G_{j}}\left(G_{k}\right) & \leqslant1-\alpha & \forall k\geqslant2
\end{align*}
 must satisfy $n\leqslant K$.
\end{assumption}

\begin{rem}
The above assumption suggests it is impossible to find very long sequences
of `independent' games that all capture some new meaningful information
about any game: in other words, the number of different skills that
have a substantial contribution to a game is finite. This natural
assumption is, for instance, easily verified in the toy model introduced
in Section \ref{subsec:toy-model} for $\alpha$ small enough and
$K=\mathcal{O}\left(\left(1-\alpha\right)/\alpha\right)$.
\end{rem}

From the Assumption \ref{assu:boundedness}, if there is an `undiscovered'
important game, there is at least one chain leading to it and we will
eventually find it, as it is impossible to `get lost' in too many
games starting from a game. Note that in practice, it is important
to ensure that $K$ is reasonably small. From the Section \ref{subsec:scope-growth},
we propose a simple greedy algorithm aimed at extracting a sequence
of games which arguably captures the same abilities as those discovered
by dynamics like the one suggested in Section \ref{subsec:meta-game-informal-dynamics-idea}.
\begin{rem}
A very short summary of the above discussion is simply: \emph{any
computable algorithm is a strategy}, so any generally capable agent
must be ready to face it, and conversely, any strategy must uncover
a family of connected games, so any fine-grained enough will uncover
them. 
\end{rem}

\subsection{\label{subsec:scope-growth}Scope Growth Algorithm}

\subsubsection{\label{subsec:simple-minded-point-of-view}Simple-Minded Point of
View}

The success of the ideas of the Novelty Search with Local Competition
algorithms together with universality considerations postulated above
(Section \ref{subsec:universality-do-details-matter}) suggest that
\emph{naive algorithms} for evolutionary processes like the ones discussed
in Section \ref{subsec:meta-game-informal-dynamics-idea} are a promising
way to search for an extended Xent Game scope. 

\subsubsection{\label{subsec:algorithm-specification}Algorithm Specification}

From the considerations made in Section \ref{subsec:evolution-based-scope-expansion}
above, we propose the greedy scope growth algorithm, which builds
an archive of games (in the terminology of the Quality-Diversity algorithms)
as a means to build Xent Measures. 

\begin{algorithm}
\begin{itemize}
\item Parameters: $\alpha_{1}>0,\alpha_{2}>0$, a meta-sampling process
$\Lambda$ (see Section \ref{subsec:meta-sampling-process}).
\item Start with an initial scope $\Sigma_{0}$. 
\item For $k=0,1,2,\ldots,N$
\begin{itemize}
\item Set $\Omega_{k}=\emptyset$.
\item For each $G$ in a randomized ordering of $\Sigma_{k}\setminus\Sigma_{k-1}$:
\begin{itemize}
\item Repeat until stop:
\begin{itemize}
\item Sample from the meta-sampling process $\Lambda$ a list of games $\mathcal{G}_{j}$
related to $G$.
\item Find the $G_{*}\in\mathcal{G}_{j}$ such that $G_{*}\rightmoon_{\alpha_{1}}G$
that achieves the smallest value of $\gamma:=\max_{\tilde{G}\in\Sigma_{k}}\mathcal{V}_{\tilde{G}}\left(G_{*}\right)$
possible.
\item If $\gamma\leq1-\alpha_{2}$: add $G_{*}$ to $\Omega_{k}$.
\item Else: stop.
\end{itemize}
\end{itemize}
\item Set $\Sigma_{k+1}=\Sigma_{k}\cup\Omega_{k}$ if $\Omega_{k}\neq\emptyset$.
\item Output the Xent Measure $\Sigma_{k+1}$.
\item Stop if $\Omega_{k}=\emptyset$. 
\end{itemize}
\end{itemize}
\caption{\label{alg:scope-growth-algorithm}Scope Growth Algorithm}
\end{algorithm}

\begin{rem}
By the Assumption \ref{assu:boundedness}, each `repeat' round is
guaranteed to stop after $K$ steps. In practice, the fact that the
list of games $\mathcal{G}_{j}$ must be sampled from an LLM will
make this stop much faster. 
\end{rem}

From our picture of the Xent Game space, we obtain the following:
\begin{claim}
From Assumptions \ref{assu:alpha-connectedness} and \ref{assu:boundedness},
for small enough $\alpha>0$, the Scope Growth Algorithm (\ref{alg:scope-growth-algorithm})
with $\alpha_{1},\alpha_{2}\leq\alpha$ covers the space of playable
Xent Games as $N\to\infty$.
\end{claim}

\begin{rem}
While we do not have explicit fairness guarantees, Algorithm \ref{alg:scope-growth-algorithm}
tends to avoid the oversampling of any region (as any new game must
avoid being $\left(1-\alpha\right)$-covered by the previous games),
thus limiting the oversampling risk. 
\end{rem}

\subsubsection{\label{subsec:meta-sampling-process}Meta-Sampling Process}

It must be noted that while the scope growth algorithm presented above
is very simple in structure, much of the implementation is left in
the choice of the meta-sampling process $\Lambda$. 

The objective of $\Lambda$ is quite clear: to provide the largest
possible sample  of games that allows the loop of the scope growth
algorithm to run for as long as possible, in a way that maximally
exhausts games. Given the structure of the Xent Games, the most useful
practical implementations of such an algorithm involve sampling random
codes in the XGL language from a dedicated model, trained to maximize
the sampling quality. As discussed in Section \ref{subsec:xgl-design-and-features},
the Xent Game structure and the XGL are designed from the beginning
with this goal in mind. 

\subsubsection{\label{subsec:numerical-simulation-exploration-in-toy-model}Numerical
Simulation: Exploration in Toy Model}

Taking the toy model defined in Section \ref{subsec:toy-model}, we
can illustrate the scope growth algorithm in a naive 2D setting (representing
games as uniform random positive points under a curve $\sqrt{x}+\sqrt{y}=1$),
by taking the meta-sampling $\Lambda$ to give the list of points
within a small neighborhood of radius $\rho$ of the game $G$. As
expected, the algorithm progressively crawls towards the Pareto efficiency
region, and ultimately covers it completely. 

\begin{figure}
\includegraphics[scale=0.2]{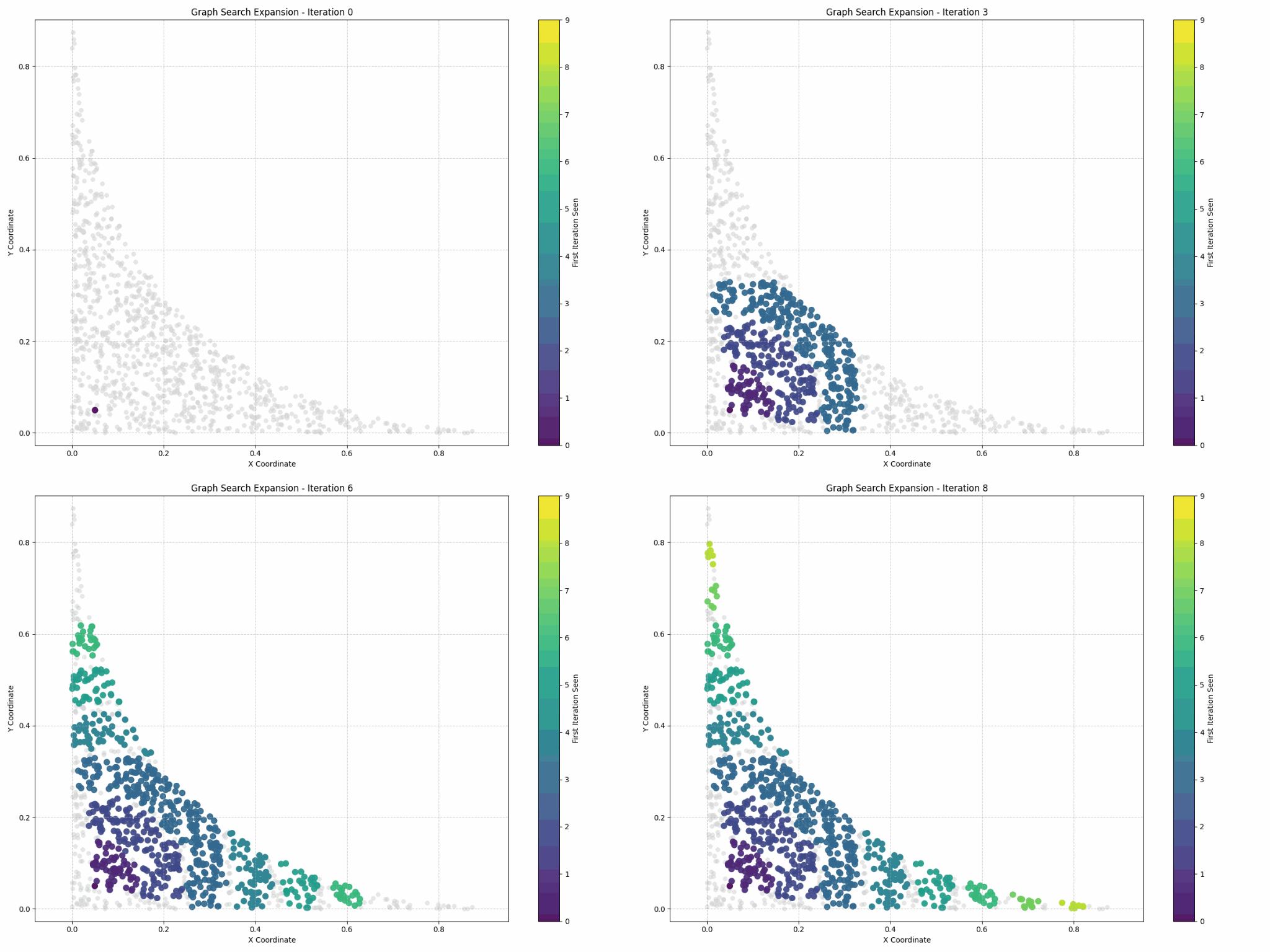}

\caption{Scope growth algorithm exploration steps $2D$ with $\alpha_{1}=0.9$,
$\alpha_{2}=0.05$, $\rho=0.1$. }

\end{figure}

\section{\label{sec:summary-discussion-and-perspectives}Summary, Discussion
and Perspectives}

\subsection{\label{subsec:summary}Summary}

Informally speaking, this paper follows a path from implicit knowledge
to general capability measures:
\begin{itemize}
\item We start from the question of implicit knowledge, and formulate a
number of concrete tasks associated with it.
\item We propose to approach these tasks using a common game-theoretic framework,
and ask the question of what is a good space of games to work with.
\item We introduce Cross-Entropy Games (Xent Games) as an answer to the
previous item, characterize this space theoretically and provide examples
of tasks covered by it.
\item We investigate the question of what it means for an LLM to play Xent
Games, propose evaluation procedures, and introduce notions of playability
and of transfer value. 
\item We propose to evaluate LLMs by making them play Xent Games, propose
a score normalization scheme for playable games, and, given a scope
(a set of playable games), propose to construct Xent Measures by considering
optimal covers in terms of transfer value. 
\item We propose to address the challenge of unbounded scope associated
with measuring general capabilities (as opposed to limiting to a finite
scope) by relying on evolutionary dynamic principles, leading us to
considerations close in philosophy to Quality-Diversity ideas.
\item Motivated by universality considerations, we propose a simple greedy
scope growth meta-algorithm as a general capability measure. 
\end{itemize}

\subsection{\label{subsec:discussion}Discussion}

A number of important points in the approach above deserve some discussion:

\subsubsection{\label{subsec:role-of-judge}Role of Judge Model}

Throughout the paper, the judge model $\mathcal{J}$ is taken as a
given. A central point is that for a sufficiently strong judge model,
the space of implicit knowledge questions becomes large enough so
that games can be played `ad infinitum' (this is similar to the fact
that once one knows the rules of chess, one can keep improving at
the game for a very long time). Still, improving the abilities of
$\mathcal{J}$ or (perhaps counter-intuitively) allowing $\mathcal{J}$
to be `controllably weak' (e.g. with certain prefix tokens) can lead
to more sophisticated games. Ultimately, the role of $\mathcal{J}$
is to provide an \emph{environment} where abilities ought to be developed.
As such, having high-quality data from the `external world' contained
in the training of the $\mathcal{J}$ model is fundamental, as the
$\mathcal{J}$ model is the only connection between the external world
and an agent playing games.

\subsubsection{\label{subsec:units-of-measure}Units of Measure}

An important theme of Sections \ref{sec:xent-game-space-model-play-properties},
\ref{sec:xent-game-measures}, and \ref{sec:evolution-in-game-space}
is to provide a clean interpretation of the quantities computed, in
particular specifying good \emph{units} of measure. While the raw
Xent Games values are measured in bits (which carries some useful
problem-specific information), the measures shift towards relative
units of work for the score normalizations with respect to a base
model (comparing number of attempts) and the transfer value (such
values are normalized by comparing token processing counts). 

Ultimately, this leads to a simple perspective on general capabilities
which is specified in terms of units of work spared compared to specific
replaying and training, assuming a common `reality' defined by a model
$\mathcal{J}$. As such, this provides a view of general capabilities
that is continuous and does not contain any explicit thresholds for
capabilities. 

\subsubsection{Connections}

Ultimately, our construction relies upon a combination of ideas which
appear in information theory, compositional game theory, transfer
learning, search, benchmarking, and open-ended evolution. While some
combinations of the above ideas are definitely well-studied, it is
interesting to note that, to the best of our knowledge, the connection
between compositional game theory and transfer learning has not been
studied at all; it is also interesting that the connection between
open-ended evolution and general intelligence has seen a recent surge
in interest \cite{aa-team,agent-gym,zhang-hu-lu-lange-clune}, but
that no connection with the question of measuring intelligence has
been provided yet. 

\subsection{\label{subsec:perspectives}Perspectives}

Building the foundations of the universality ideas (Section \ref{subsec:universality-do-details-matter})
appears to be one of the most interesting theoretical challenges associated
with our results: while usually the evolutionary ideas are often proposed
as a very imperfect means towards an end, it is possible that as far
as building general capability measures via meta-games, the imperfection
disappears for game-theoretic reasons.

\subsubsection{Implementation}

The most enticing practical challenge associated with the vision proposed
in this paper is its implementation at scale: this involves in particular
the determination and study of families of Xent Measures generated
according to the ideas proposed in Section \ref{sec:evolution-in-game-space},
including, in particular, efficient means to perform the game meta-sampling
and hyperparameter selection (e.g. the $\alpha$ values). 

In parallel, an interesting question is to investigate the potential
of hand-crafted Xent Measures (i.e. for a finite set of hand-crafted
games) as restricted benchmarks: in spite of their inherently narrow
scopes, these already satisfy a number of interesting properties (in
particular, of not relying on a private dataset, and being generally
hard to overfit on) that make them promising contenders to measure
the abilities of models in a trusted and credible fashion. 

\subsubsection{Connection with Human-Centric Tasks}

While the questions raised in Section \ref{subsec:implicit-capabilities}
to motivate the investigation of the implicit knowledge have led us
to games that are often directly inspired by an attempt to answer
these questions (as illustrated in Section \ref{subsec:xega-space-characterization}),
to make the \emph{solutions to the games} directly relevant to human-centric
questions (e.g. to make good solutions of reverse prompting game \ref{exa:reverse-prompt}
useful as human-directed summaries) likely requires some engineering
(e.g. some appropriate pre-prompting or fine-tuning of the judge model
focus on certain modes of expression, adding constraints, etc.). 

From a capability measure perspective, by the very design of Xent
game space (and the transfer learning considerations), it is reasonable
to expect that models that display competence at Xent games will be
capable to find solutions to their human-centric counterparts; at
the same time, formulating precisely what the latter are is an important
question. 

\subsubsection{\label{subsec:synthetic-data}Synthetic Data}

An important question in the field of Language Models is that of synthetic
data (see e.g. \cite{data-augmentation,best-practices-synth-data})
for LLM training. An interesting question is to determine how much
value can be learned by pre-training LLMs on synthetic data coming
from gameplay of Xent Games. The long-context window and other challenging
gameplay aspects could in principle lead pre-trained models to learn
deeper and more nuanced representations of information, enabling them
to capture subtle notions associated with implicit knowledge tasks.

\subsubsection{\label{subsec:curriculum-learning}Curriculum Learning}

An interesting direction is the construction of curriculum-learning
algorithms: the approach we have proposed so far presupposes the presence
of agents whose abilities are to be measured, but leaves aside the
question of how such abilities ought to be discovered. A promising
path is to leverage evolutionary ideas such as the ones outlined in
the present paper to build such curricula, i.e. to identify the most
relevant games to learn to play in order to develop general abilities. 

\subsubsection{Putting Things Together}

All in all, in the setting of Xent Games, we have LLMs play several
distinct roles. Besides the base models, used for normalization, we
have LLMs:
\begin{itemize}
\item As judge models
\item As models for the NPC players in the games
\item As models to generate the game maps
\item As meta-sampling models to generate games
\end{itemize}
In principle, these models can all be the same, and they can evolve
over time. If one leverages Xent Games as means to train new models
(as suggested in Sections \ref{subsec:synthetic-data} and \ref{subsec:curriculum-learning}
above), then this can lead to a self-improvement loop of LLMs, which
is an interesting alley to pursue: this is one way in which evolutionary
ideas could (in principle) lead to an augmentation of capabilities
for LLMs in new, open-ended dimensions.

\end{document}